\documentclass{article}

\usepackage[nonatbib, preprint]{neurips_2021}
\usepackage[square,numbers,compress]{natbib}
\usepackage[utf8]{inputenc} 
\usepackage[T1]{fontenc}    
\usepackage{url}            
\usepackage{booktabs}       
\usepackage{amsfonts}       
\usepackage{nicefrac}       
\usepackage{microtype}      
\usepackage{mathtools}      

\RequirePackage{algorithm}
\usepackage[noend]{algpseudocode}               

\renewcommand{\citet}{\citep}

\usepackage{caption}
\usepackage{subcaption}
\usepackage{listings}
\usepackage{color}
\usepackage{amsmath,amssymb}
\usepackage{amsthm}
\usepackage{mathabx}
\usepackage{wrapfig}
\usepackage{tikz}
\usetikzlibrary{positioning}
\usepackage{algorithm}
\usepackage{graphicx}
\usepackage[backref=page]{hyperref}       
\renewcommand*{\backref}[1]{}
\renewcommand*{\backrefalt}[4]{{\footnotesize [%
    \ifcase #1 Not cited.%
	\or Cited on page~#2%
	\else Cited on pages #2%
	\fi%
]}}

\hypersetup{
    colorlinks=true,
    linkcolor=blue,
    filecolor=magenta,      
    urlcolor=cyan,
    citecolor=blue,
}

\usepackage{siunitx}

\DeclareMathOperator*{\argmin}{arg\,min}
\newcommand{\argminSide}{\textnormal{arg\,min}}
\newcommand{\defeq}{\vcentcolon=}

\usepackage{xcolor}

\newcommand{\new}[1]{#1}

\newcommand{\loss}{\mathcal{L}}  

\newcommand{\outSymbol}{A}
\newcommand{\inSymbol}{B}

\newcommand{\paramSymbol}{\boldsymbol{\theta}}
\newcommand{\outParam}{\paramSymbol_{\!\outSymbol}}
\newcommand{\inParam}{\paramSymbol_{\!\inSymbol}}
\newcommand{\bothParam}{\boldsymbol{\omega}}
\newcommand{\outLoss}{\loss_{\!\outSymbol}}
\newcommand{\inLoss}{\loss_{\!\inSymbol}}
\newcommand{\gradSymbol}{\boldsymbol{g}}
\newcommand{\bothGrad}{\hat{\gradSymbol}}
\newcommand{\outGrad}{\gradSymbol_{\!\outSymbol}}
\newcommand{\inGrad}{\gradSymbol_{\!\inSymbol}}
\newcommand{\momCoeff}{\beta}
\newcommand{\momVal}{\boldsymbol{\mu}}
\newcommand{\lr}{\alpha}

\newcommand{\spectralRadius}{\rho}
\DeclareMathOperator*{\spectrum}{Sp}

\newcommand{\condition}{\kappa}
\newcommand{\identity}{\boldsymbol{I}}
\newcommand{\fixedPointOp}{\boldsymbol{F}_{\lr, \momCoeff}}
\newcommand{\dynamics}{\boldsymbol{R}}
\newcommand{\dimSymbol}{d}
\newcommand{\inDim}{\dimSymbol_{\inSymbol}}
\newcommand{\outDim}{\dimSymbol_{\outSymbol}}
\newcommand{\bothDim}{\dimSymbol}
\newcommand{\eigval}{\lambda}

\newcommand{\transpose}{\top}

\newcommand\newsubcap[1]{\phantomcaption%
       \caption*{\figurename~\thefigure(\thesubfigure): #1}}

\newcommand{\eqspace}{\,\,\,\,\,}

\newcommand{\resultName}{Corollary}
\newtheorem{lemma}{Lemma}

\usepackage{thmtools}
\usepackage{thm-restate}

\newcommand{\localTitle}{\vspace{-0.000\textheight}Complex Momentum for Optimization in Games\vspace{-0.000\textheight}}

\title{\localTitle}

\author{%
  Jonathan Lorraine$\vphantom{e}^{1, 2}$\hspace{0.02\textwidth} David Acuna$\vphantom{e}^{1, 2, 3}$\hspace{0.02\textwidth} Paul Vicol$\vphantom{e}^{1, 2}$\hspace{0.02\textwidth} David Duvenaud$\vphantom{e}^{1, 2}$\\
  University of Toronto$\vphantom{e}^{1}$\hspace{0.02\textwidth} Vector Institute$\vphantom{e}^{2}$\hspace{0.02\textwidth} NVIDIA$\vphantom{e}^{3}$\\
  \texttt{\{lorraine, davidj, pvicol, duvenaud\}@cs.toronto.edu}
}

\begin{document}
    \maketitle
    \vspace{-0.03\textheight}
    \begin{abstract}
        \vspace{-0.015\textheight}
        We generalize gradient descent with momentum for optimization in differentiable games to have complex-valued momentum.
        We give theoretical motivation for our method by proving convergence on bilinear zero-sum games for simultaneous and alternating updates.
        Our method gives real-valued parameter updates, making it a drop-in replacement for standard optimizers.
        We empirically demonstrate that complex-valued momentum can improve convergence in realistic adversarial games---like generative adversarial networks---by showing we can find better solutions with an almost identical computational cost.
        We also show a practical generalization to a complex-valued Adam variant, which we use to train BigGAN to better inception scores on CIFAR-10.
    \end{abstract}
    \vspace{-0.025\textheight}
    \section{Introduction}\label{sec:introduction}
    \vspace{-0.01\textheight}
        Gradient-based optimization has been critical for the success of machine learning, updating a single set of parameters to minimize a single loss.
        A growing number of applications require learning in games, which generalize single-objective optimization.
        Common examples are GANs~\citep{goodfellow2014generative}, actor-critic models~\citep{pfau2016connecting}, curriculum learning~\citep{baker2019emergent, balduzzi2019open, sukhbaatar2017intrinsic}, hyperparameter optimization~\citep{lorraine2018stochastic, lorraine2019optimizing, mackay2019self, raghu2020teaching}, adversarial examples~\citep{bose2020adversarial, yuan2019adversarial}, learning models~\citep{rajeswaran2020game, abachi2020policy,  baconlagrangian}, domain adversarial adaptation~\citep{acuna2021fdomainadversarial}, neural architecture search~\citep{grathwohl2018gradient, adam2019understanding}, and meta-learning~\citep{ren2018meta, ren2020flexible}.
        
        Games consist of multiple players, each with parameters and objectives.
        We often want solutions where no player gains from changing their strategy unilaterally, e.g., Nash equilibria~\citep{morgenstern1953theory} or Stackelberg equilibria~\citep{von2010market}.
        Classical gradient-based learning often fails to find these equilibria due to rotational dynamics~\citep{berard2019closer}.
        %
        Numerous saddle point finding algorithms for zero-sum games have been proposed~\citep{arrow1958studies, freund1999adaptive}.
        \citet{gidel2018negative} generalizes GD with momentum to games, showing we can use a negative momentum to converge if the eigenvalues of the Jacobian of the gradient vector field have a large imaginary part.
        We use the terminology in \citet{gidel2018negative} and say \textit{(purely) cooperative or adversarial games} for games with (purely) real or imaginary eigenvalues.
        Setups like GANs are not purely adversarial, but rather have both \emph{purely cooperative and adversarial eigenspaces} -- i.e., eigenspaces with purely real or imaginary eigenvalues.
        In cooperative eigenspaces, the players do not interfere with each other.
        
        We want solutions that converge with simultaneous and alternating updates in purely adversarial games -- a setup where existing momentum methods fail.
        Also, we want solutions that are robust to different mixtures of adversarial and cooperative eigenspaces, because this depends on the games eigendecomposition which can be intractable.
        To solve this we unify and generalize existing momentum methods~\citep{lucas2018aggregated, gidel2018negative} to recurrently linked momentum -- a setup with multiple recurrently linked momentum buffers with potentially negative coefficients shown in Figure~\ref{fig:recurrent_comp_graph}.
        
        We show that selecting two of these recurrently linked buffers with appropriate momentum coefficients can be interpreted as the real and imaginary parts of a single \emph{complex buffer} and complex momentum coefficient -- see Figure~\ref{fig:cm_comp_graph}.
        This setup (a) allows us to converge in adversarial games with simultaneous updates, (b) only introduces one new optimizer parameter -- the phase or $\arg$ of our momentum, (c) allows us to gain intuitions via complex analysis, (d) is trivial to implement in libraries supporting complex arithmetic, and (e) robustly converges for different eigenspace mixtures.
        
        
        Intuitively, our complex buffer stores historical gradient information, oscillating between adding or subtracting at a frequency dictated by the momentum coefficient.
        Classical momentum only adds gradients, and negative momentum changes between adding or subtracting each iteration, while we oscillate at an arbitrary (fixed) frequency -- see Figure~\ref{fig:frequency_vis}.
        This reduces rotational dynamics during training by canceling out opposing updates.
        
        
        \vspace{-0.00\textheight}
        \subsection*{Contributions}\label{subsec:contributions}
            \begin{itemize}
            	\item We provide generalizations and variants of classical~\citep{polyak1964some, nesterov1983method, sutskever2013importance}, negative~\citep{gidel2018negative, zhang2020suboptimality}, and aggregated~\citep{lucas2018aggregated} momentum for learning in differentiable games.
            	\item We show our methods converges on adversarial games -- including bilinear zero-sum games and the Dirac-GAN -- with simultaneous and alternating updates.
            	\item We illustrate a robustness during optimization, converging faster and over a larger range of mixtures of cooperative and adversarial games than existing first-order methods.
            	\item We give a practical extension of our method to a complex-valued Adam~\citep{kingma2014adam} variant, which we use to train a BigGAN~\citep{brock2018large} on CIFAR-10, improving \citep{brock2018large}'s inception scores.
            	\vspace{-0.005\textheight}
            \end{itemize}
        %
    \begin{wrapfigure}[20]{R}{.56\textwidth}
        \centering
        \begin{minipage}{.55\textwidth}
            \centering
            \vspace{-0.045\textheight}
            \lstset{
            language=Python,
            basicstyle=\footnotesize\ttfamily,
            alsoletter={., +},
            keywords=[2]{jnp.real, .3j},
            keywordstyle=[2]{\color{green!80!black}}
            }
            \textbf{Actual JAX implementation: changes in {\color{green}green}}\par\medskip
            \vspace{-0.015\textheight}
            \captionsetup{type=figure}
            \begin{lstlisting}
mass = .8 + .3j
def momentum(step_size, mass):
 ...
 def update(i, g, state):
  x, velocity = state
  velocity = mass * velocity + g
  x=x-jnp.real(step_size(i)*velocity)
  return x, velocity
 ...
            \end{lstlisting}
            %
            
        \end{minipage}
        \vspace{-0.02\textheight}
        \hspace{1cm}
        \caption{
            How to modify JAX's \ref{eq:single_objective_gd} with momentum \href{https://jax.readthedocs.io/en/latest/_modules/jax/experimental/optimizers.html}{{\color{blue}here}} to use complex momentum.
            The only changes are in {\color{green}green}.
            $\texttt{jnp.real}$ gets the real part of \texttt{step\_size} times the momentum buffer (called $\texttt{velocity}$ here).
            We use a complex \texttt{mass} for our method in this case $\momCoeff = |\momCoeff|\exp(i\arg(\momCoeff)) = \num{0.9}\exp(i\nicefrac{\pi}{\num{8}}) \approx .8 + .3 i$.
        }
        \label{fig:jax_code_change}
    \end{wrapfigure}
    
    \vspace{-0.01\textheight}
    \section{Background}\label{sec:background}
    \vspace{-0.01\textheight}
            Appendix Table~\ref{tab:TableOfNotation} summarizes our notation.
            Consider the optimization problem:
            \begin{equation}\label{eq:single_objective_opt} 
                    \paramSymbol^{*} \defeq \argminSide_{\paramSymbol} \loss(\paramSymbol) 
            \end{equation}
            We can find local minima of loss $\loss$ using (stochastic) gradient descent with step size $\lr$.
            We denote the loss gradient at parameters $\paramSymbol^{j}$ by $\gradSymbol^j \!\!\defeq\!\! \gradSymbol(\paramSymbol^j) \!\!\defeq\!\! \smash{\left. \nabla_{\paramSymbol} \loss(\paramSymbol) \right|_{\paramSymbol^{j}}}$.
            %
            \begin{equation}\label{eq:single_objective_gd}\tag{SGD}
                \smash{\paramSymbol^{j\!+\!1} = \paramSymbol^{j} - \alpha \gradSymbol^j}
            \end{equation}
            Momentum can generalize \ref{eq:single_objective_gd}.
            For example, Polyak's Heavy Ball~\citep{polyak1964some}:
            \begin{align}\label{eq:single_objective_polyak}
                \smash{\paramSymbol^{j\!+\!1} = \paramSymbol^{j} - \lr \gradSymbol^j + \momCoeff (\paramSymbol^{j} - \paramSymbol^{j-1})}
            \end{align}
            Which can be equivalently written with momentum buffer $\smash{\momVal^{j} = \nicefrac{(\paramSymbol^{j} - \paramSymbol^{j-1})}{\lr}}$.
            \begin{equation}\label{eq:single_objective_sgdm}\tag{SGDm}
                \smash{\momVal^{j\!+\!1} = \momCoeff \momVal^{j} - \gradSymbol^j,\quad\quad
                \paramSymbol^{j\!+\!1} = \paramSymbol^{j} + \lr \momVal^{j\!+\!1}}
            \end{equation}
            We can also generalize \ref{eq:single_objective_sgdm} to aggregated momentum~\citep{lucas2018aggregated}, shown in Appendix Algorithm~\ref{alg:aggmo}.
        \vspace{-0.01\textheight}
        \subsection{Game Formulations}
            Another class of problems is learning in \textit{games}, which includes problems like generative adversarial networks (GANs)~\citep{goodfellow2014generative}.
            We focus on $\num{2}$-player games ---with players denoted by $A$ and $B$---where each player minimizes their loss $\loss_A, \loss_B$ with their parameters $\paramSymbol_A \in \mathbb{R}^{\dimSymbol_{A}}$, $\paramSymbol_B \in \mathbb{R}^{\dimSymbol_{B}}$\!\!.
            Solutions to \num{2}-player games -- which are assumed unique for simplicity -- can be defined as:
            \begin{align}\label{eq:nash_equilibrium}
                \smash{\outParam^{*} \!\defeq\! \argminSide_{\outParam}\! \outLoss(\outParam,\! \inParam^{*}),\eqspace
                \inParam^{*} \!\defeq\! \argminSide_{\inParam}\! \inLoss(\outParam^{*},\! \inParam)}
            \end{align}
            In deep learning, losses are non-convex with many parameters, so we often focus on finding local solutions.
            If we have a player ordering, then we have a Stackelberg game.
            For example, in GANs, the generator is the leader, and the discriminator is the follower.
            In hyperparameter optimization, the hyperparameters are the leader, and the network parameters are the follower.
            If \smash{$\inParam^*(\outParam)$} denotes player $\inSymbol$'s best-response function, then Stackelberg game solutions can be defined as:
            \begin{align}\label{eq:stackelberg_equilibrium}
                \begin{smash}
                    \outParam^{*} \!\defeq \argminSide_{\outParam}\! \outLoss(\outParam,\! \inParam^{*}(\outParam)),\eqspace
                    \inParam^{*}(\outParam) \!\defeq \argminSide_{\inParam}\! \inLoss(\outParam,\! \inParam)
                \end{smash}
            \end{align}
            If $\outLoss$ and $\inLoss$ are differentiable in $\outParam$ and $\inParam$ we say the game is differentiable.
            We may be able to approximately find \smash{$\outParam^{*}$} efficiently if we can do \ref{eq:single_objective_gd} on:
            \begin{equation}
                \outLoss^*(\outParam) \defeq \outLoss(\outParam, \inParam^{*}(\outParam))
            \end{equation}
            Unfortunately, SGD would require computing $\nicefrac{d \outLoss^*}{d \outParam}$, which often requires $\nicefrac{d \inParam^{*}}{d \outParam}$, but \smash{$\inParam^{*}(\outParam)$} and its Jacobian are typically intractable.
            A common optimization algorithm to analyze for finding solutions is simultaneous SGD (\ref{eq:multi_objective_gd_long}) -- sometimes called gradient descent ascent for zero-sum games -- where \smash{$\outGrad^j \defeq \outGrad(\outParam^j, \inParam^j)$} and \smash{$\inGrad^j \defeq \inGrad(\outParam^j, \inParam^j)$} are estimators for \smash{$\left. \smash{\nabla_{\outParam}} \outLoss\right|_{\outParam^j, \inParam^j}$ and $\left. \smash{\nabla_{\inParam} \inLoss} \right|_{\outParam^j, \inParam^j}$}: 
            %
            \begin{align}\label{eq:multi_objective_gd_long}\tag{SimSGD}
                \smash{
                \outParam^{j\!+\!1} = \outParam^{j} - \lr \outGrad^j,\quad
                \inParam^{j\!+\!1} = \inParam^{j} - \lr \inGrad^j
                }
            \end{align}
            We simplify notation with the concatenated or joint-parameters $\smash{\bothParam \!\defeq\! [\outParam, \inParam] \!\in\! \mathbb{R}^{\bothDim}}$ and the joint-gradient vector field $\bothGrad: \mathbb{R}^{\bothDim} \to \mathbb{R}^{\bothDim}$, which at the $j^{th}$ iteration is the joint-gradient denoted:
            \begin{equation}
                \bothGrad^j \defeq \bothGrad(\bothParam^j) \defeq [\outGrad(\bothParam^j), \inGrad(\bothParam^j)] = [\outGrad^j, \inGrad^j]
            \end{equation}
            We extend to $n$-player games by treating $\bothParam$ and $\bothGrad$ as concatenations of the players' parameters and loss gradients, allowing for a concise expression of the \ref{eq:multi_objective_gd_long} update with momentum (\ref{eq:multi_objective_sgdm}):
            \begin{align}\label{eq:multi_objective_sgdm}\tag{SimSGDm}
                \smash{
                \momVal^{j\!+\!1} = \momCoeff \momVal^{j} - \bothGrad^j,\quad\quad
                \bothParam^{j\!+\!1} = \bothParam^{j} + \lr \momVal^{j\!+\!1}
                }
            \end{align}
            \begin{wrapfigure}{r}{.52\textwidth}
                    \vspace{-0.02\textheight}
                \end{wrapfigure}
            \citet{gidel2018negative} show classical momentum choices of $\momCoeff \in [\num{0}, \num{1})$ do not improve solution speed over \ref{eq:multi_objective_gd_long} in some games, while negative momentum helps if the Jacobian of the joint-gradient vector field $\nabla_{\bothParam} \bothGrad$ has complex eigenvalues.
            Thus, for purely adversarial games with imaginary eigenvalues, any non-negative momentum and step size will not converge.
            For cooperative games -- i.e., minimization -- $\nabla_{\bothParam} \bothGrad$ has strictly real eigenvalues because it is a losses Hessian, so classical momentum works well.
            %
            %
             \begin{figure}[b]
                \vspace{-0.025\textheight}
                \centering
                \begin{subfigure}{.16\linewidth}
                    \begin{tikzpicture}[
                            > = stealth, 
                            shorten > = 1pt, 
                            auto,
                            node distance = 1.8cm, 
                            semithick 
                        ]
                        \tikzstyle{state}=[
                            draw = black,
                            thick,
                            fill = white,
                            minimum size = 4mm
                        ]
                
                        \node[state] (gNext) {gradient};
                        \node[state] (more) [yshift=.6cm, below of=gNext] {$\momVal$};
                        \node[state] (wNext) [yshift=.6cm, below of=more] {update};
                        
                        \path[->]
                            (gNext) 
                                    edge                node                [xshift=-0cm,]{}              (more)
                            (more)    
                                    edge                node                [xshift=0cm, yshift=0cm]{$\lr$}          (wNext)
                                    edge[loop left]    node                {$\momCoeff$}                             (re)
                          ;
                    \end{tikzpicture}
                    \caption{Classical~\citep{polyak1964some}}
                \end{subfigure}
                \begin{subfigure}{.27\linewidth}
                    \begin{tikzpicture}[
                            > = stealth, 
                            shorten > = 1pt, 
                            auto,
                            node distance = 1.8cm, 
                            semithick 
                        ]
                        \tikzstyle{state}=[
                            draw = black,
                            thick,
                            fill = white,
                            minimum size = 4mm
                        ]
                
                        \node[state] (gNext) {gradient};
                        \node[state] (re) [below left of=gNext] {$\momVal_{(1)}$};
                        \node[state] (more) [xshift=-.55cm, right of=re] {\dots};
                        \node[state] (im) [below right of=gNext] {$\momVal_{(n)}$};
                        \node[state] (wNext) [below right of=re] {update};
                        
                        \path[->]
                            (re)    
                                    edge                node                [xshift=-.7cm, yshift=-.45cm]{$\lr_{(1)}$}          (wNext)
                                    edge[loop above]    node                {$\momCoeff_{(1)}$}                             (re)
                            (im)    
                                    edge                node                {$\lr_{(n)}$}                                       (wNext)
                                    edge[loop above]    node                {$\momCoeff_{(n)}$}                             (im)
                            (gNext) edge                node                [xshift=-.6cm, yshift=.65cm]{} (re)
                                    edge                node                [xshift=-.2cm,]{}              (im)
                                    edge                node                [xshift=-0cm,]{}              (more)
                            (more)    
                                    edge                node                [xshift=0cm, yshift=-.25cm]{}          (wNext)
                                    edge[loop right]    node                [xshift=0cm]{}                             (re)
                          ;
                    \end{tikzpicture}
                    \caption{Aggregated~\citep{lucas2018aggregated}}
                    \label{fig:aggmo_comp_graph}
                \end{subfigure}
                \begin{subfigure}{.27\linewidth}
                    \begin{tikzpicture}[
                            > = stealth, 
                            shorten > = 1pt, 
                            auto,
                            node distance = 1.8cm, 
                            semithick 
                        ]
                        \tikzstyle{state}=[
                            draw = black,
                            thick,
                            fill = white,
                            minimum size = 4mm
                        ]
                
                        \node[state] (gNext) {gradient};
                        \node[state] (re) [below left of=gNext] {$\momVal_{(1)}$};
                        \node[state] (im) [below right of=gNext] {$\momVal_{(2)}$};
                        \node[state] (wNext) [below right of=re] {update};
                        
                        \path[->]
                            (re)    edge[bend left]     node                [yshift=-.1cm]{$\momCoeff_{(1,2)}$}                             (im)
                                    edge                node                [xshift=-.7cm, yshift=-.45cm]{$\lr_{(1)}$}          (wNext)
                                    edge[loop above]    node                {$\momCoeff_{(1,1)}$}                             (re)
                            (im)    edge[bend left]     node                [xshift=0cm, yshift=.65cm]{$\momCoeff_{(2,1)}$}   (re)
                                    edge                node                {$\lr_{(2)}$}                                       (wNext)
                                    edge[loop above]    node                {$\momCoeff_{(2,2)}$}                             (im)
                            (gNext) edge                node                [xshift=-.6cm, yshift=.65cm]{} (re)
                                    edge                node                [xshift=-.2cm,]{}              (im)
                          ;
                    \end{tikzpicture}
                    \caption{Recurrently linked (new)}
                    \label{fig:recurrent_comp_graph}
                \end{subfigure}
                \begin{subfigure}{.27\linewidth}
                    \begin{tikzpicture}[
                            > = stealth, 
                            shorten > = 1pt, 
                            auto,
                            node distance = 1.8cm, 
                            semithick 
                        ]
                        \tikzstyle{state}=[
                            draw = black,
                            thick,
                            fill = white,
                            minimum size = 4mm
                        ]
                
                        \node[state] (gNext) {gradient};
                        \node[state] (re) [below left of=gNext] {$\Re(\momVal$)};
                        \node[state] (im) [below right of=gNext] {$\Im(\momVal$)};
                        \node[state] (wNext) [below right of=re] {update};
                        
                        \path[->]
                            (re)    edge[bend left]     node                [yshift=-.1cm]{${\color{red}\Im(\momCoeff)}$}                 (im)
                                    edge                node                [xshift=-.5cm, yshift=-.45cm]{$\lr$}     (wNext)
                                    edge[loop above]    node                {{\color{blue}$\Re(\momCoeff)$}}                (re)
                            (im)    edge[bend left]     node                 [xshift=0cm, yshift=.65cm]{$-{\color{red}\Im(\momCoeff)}$}  (re)
                                    edge[loop above]    node                {{\color{blue}$\Re(\momCoeff)$}}                (im)
                            (gNext) edge                node                {}                                              (re)
                          ;
                    \end{tikzpicture}
                    \caption{Complex (ours)}
                    \label{fig:cm_comp_graph}
                \end{subfigure}
                \vspace{-0.005\textheight}
                \caption{
                        \new{
                        We show computational diagrams for momentum variants simultaneously updating all players parameters, which update the momentum buffers $\momVal$ at iteration $j\!+\!1$ with coefficient $\momCoeff$ via \smash{$\momVal^{j\!+\!1} \!=\! (\momCoeff \momVal^{j} - $gradient$)$}.
                        Our parameter update is a linear combination of the momentum buffers weighted by step sizes $\lr$.
                        \emph{(a)} Classical momentum~\citep{polyak1964some, sutskever2013importance}, with a single buffer and coefficient $\momCoeff \in [0, 1)$.
                        \emph{(b)} Aggregated momentum~\citep{lucas2018aggregated} which adds multiple buffers with different coefficients.
                        \emph{(c)} Recurrently linked momentum, which adds cross-buffer coefficients and updates the buffers with \smash{$\momVal_{(k)}^{j+1} \!=\! (\sum_l \momCoeff_{(l, k)}\momVal_{(l)}^{j} - $gradient$)$}.
                        We allow $\momCoeff_{(l, k)}$ to be negative like negative momentum~\citep{gidel2018negative} for solutions with simultaneous updates in adversarial games.
                        \emph{(d)} Complex momentum is a special case of recurrently linked momentum with two buffers and \smash{${\color{blue}\momCoeff_{(1,1)}}\!=\!{\color{blue}\momCoeff_{(2,2)}} \!=\! {\color{blue}\Re(\momCoeff)}$, ${\color{red}\momCoeff_{(1,2)}} \!=\! -{\color{red}\momCoeff_{(2,1)}} \!=\! {\color{red}\Im(\momCoeff)}$}.
                        Analyzing other recurrently linked momentum setups is an open problem.
                        }
                }\label{fig:scalar_complex}
            \end{figure}
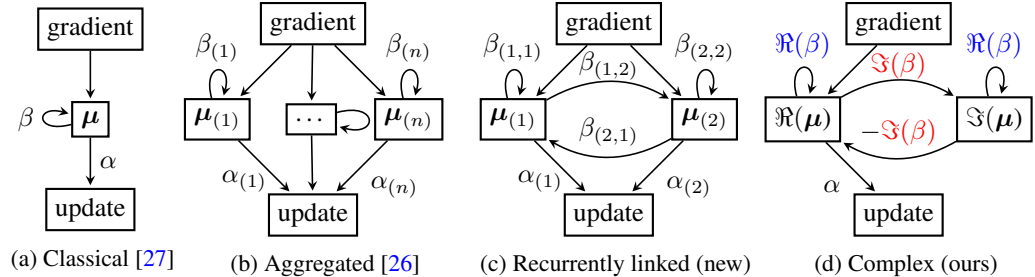
            \new{
            \vspace{-0.0145\textheight}
            \subsection{Limitations of Existing Methods}\label{sec:limitation_existing}
            \vspace{-0.0125\textheight}
                \textbf{Higher-order:} Methods using higher-order gradients are often harder to parallelize across GPUs,~\citep{osawa2019large}, get attracted to bad saddle points~\citep{mescheder2017numerics}, require estimators for inverse Hessians~\citep{schafer2019competitive, wang2019solving}, are complicated to implement, have numerous optimizer parameters, and can be more expensive in iteration and memory cost~\citep{hemmat2020lead, wang2019solving, schafer2019competitive, schafer2020competitive, czarnecki2020real, zhang2020newton}.
                Instead, we focus on first-order methods.
                
                \textbf{First-order:} Some first-order methods such as extragradient~\citep{korpelevich1976extragradient} require a second, costly, gradient evaluation per step.
                Similarly, methods alternating player updates are bottlenecked by waiting until after the first player's gradient is used to evaluate the second player's gradient.
                But, many deep learning setups can parallelize computation of both players' gradients, making alternating updates effectively cost another gradient evaluation.
                We want a method which updates with the effective cost of one gradient evaluation.
                Also, simultaneous updates are a standard choice in some settings~\citep{acuna2021fdomainadversarial}.
                
                \textbf{Robust convergence:} We want our method to converge in purely adversarial game's with simultaneous updates -- a setup where existing momentum methods fail~\citep{gidel2018negative}.
                Furthermore, computing a games eigendecomposition is often infeasibly expensive, so we want methods that robustly converge over different mixtures of adversarial and cooperative eigenspaces.
                We are particularly interested in eigenspace mixtures that that are relevant during GAN training -- see Figure~\ref{fig:gan_spectrum_main} and Appendix Figure~\ref{fig:gan_decomp}.
            
            \vspace{-0.0145\textheight}
            \subsection{Coming up with our Method}
            \vspace{-0.0125\textheight}
                \textbf{Combining existing methods:}
                Given the preceding limitations, we would like a robust first-order method using a single, simultaneous gradient evaluation.
                We looked at combining aggregated~\citep{lucas2018aggregated} with negative~\citep{gidel2018negative} momentum by allowing negative coefficients, because these methods are first-order and use a single gradient evaluation -- see Figure~\ref{fig:aggmo_comp_graph}.
                Also, aggregated momentum provides robustness during optimization by converging quickly on problems with wide range of conditioning, while negative momentum works in adversarial setups.
                We hoped to combine their benefits, gaining robustness to different mixtures of adversarial and cooperative eigenspaces.
                However, with this setup we could not find solutions that converge with simultaneous updates in purely adversarial games.
                
                \textbf{Generalize to allow solutions:}
                We generalized the setup to allow recurrent connections between momentum buffers, with potentially negative coefficients -- see Figure~\ref{fig:recurrent_comp_graph} and Appendix Algorithm~\ref{alg:recurrent_momentum}.
                There are optimizer parameters so this converges with simultaneous updates in purely adversarial games, while being first-order with a single gradient evaluation -- see \resultName~\ref{thm:theorem_existence}.
                However, in general, this setup could introduce many optimizer parameters, have unintuitive behavior, and not be amenable to analysis.
                So, we choose a special case of this method to help solve these problems.
                
                \textbf{A simple solution:} With two momentum buffers and correctly chosen recurrent weights, we can interpret our buffers as the real and imaginary part of one complex buffer -- see Figure~\ref{fig:cm_comp_graph}.
                This method is (a) capable of converging in purely adversarial games with simultaneous updates -- \resultName~\ref{thm:theorem_existence}, (b) only introduces one new optimizer parameter -- the phase of the momentum coefficient, (c) is tractable to analyze and have intuitions for with Euler's formula -- ex., Eq. (\ref{eq:polar_complex_update}), (d) is trivial to implement in libraries supporting complex arithmetic -- see Figure~\ref{fig:jax_code_change}, and (e) can be robust to games with different mixtures of cooperative and adversarial eigenspaces -- see Figure~\ref{fig:partial_coop_phases}.
            }
        %
            \begin{figure}[b]
                \vspace{-0.04\textheight}
                \centering
                    \begin{tikzpicture}
                        \centering
                        \node (img1){\includegraphics[trim={.1cm 1.cm 1.0cm .35cm},clip,width=.47\linewidth]{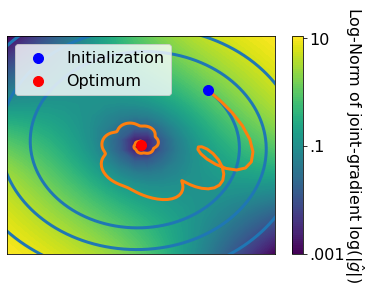}};
                        \node[left=of img1, rotate=90, node distance=0cm, xshift=.9cm, yshift=-.9cm, font=\color{black}] {Discriminator};
                        \node[below=of img1, node distance=0cm, xshift=-.7cm, yshift=1.2cm,font=\color{black}] {Generator};
                        \node[right=of img1, rotate=270, node distance=0cm, xshift=-1.8cm, yshift=-1.25cm, font=\color{black}] {Norm of joint-gradient $\|\bothGrad\|$};
                        
                        
                        \node (img3)[right=of img1, node distance=0cm, xshift=-.75cm, yshift=-.2cm, font=\color{black}]{\includegraphics[trim={.8cm .8cm .25cm .2cm},clip,width=.45\linewidth]{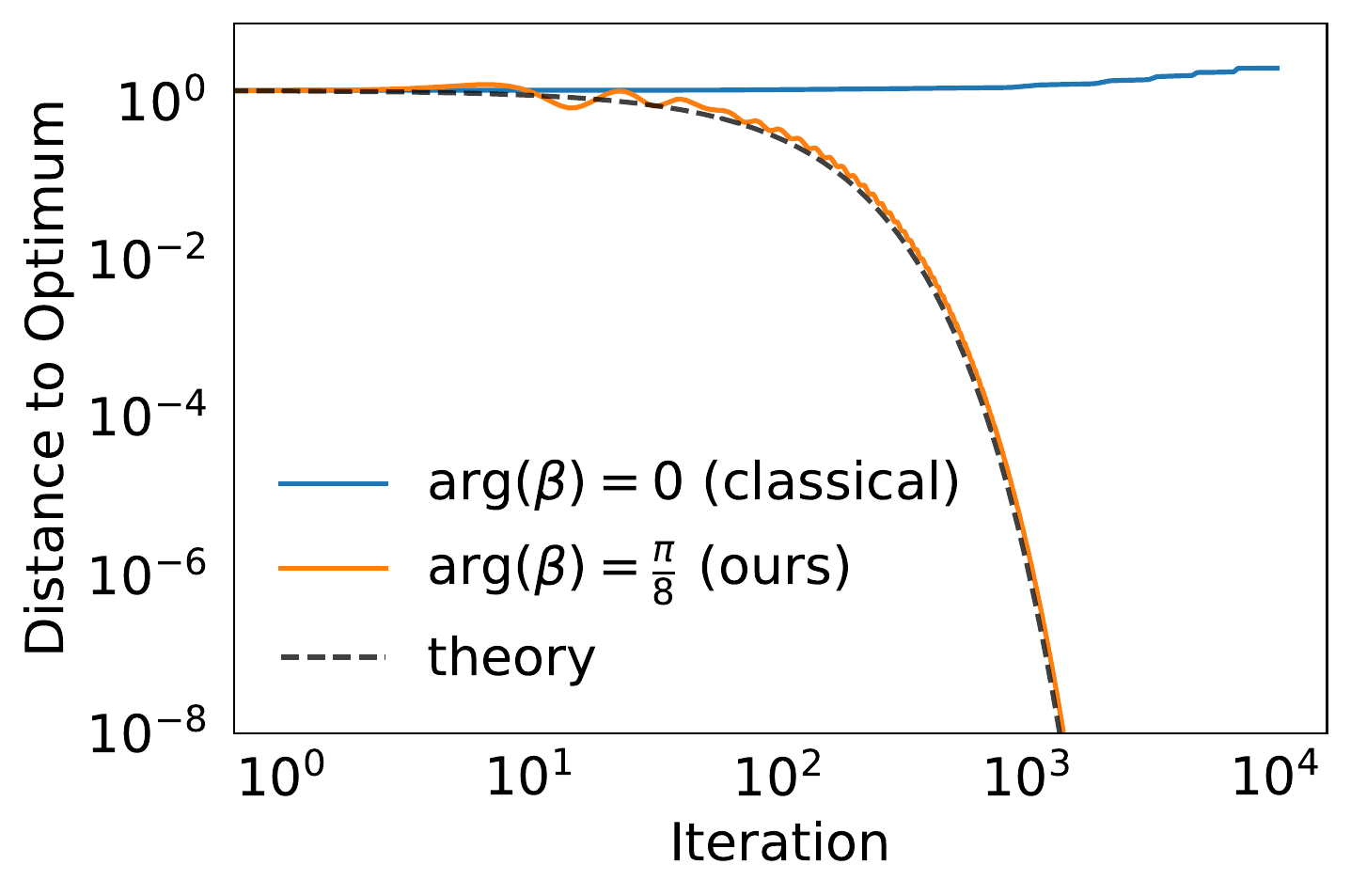}};
                        \node[left=of img3, rotate=90, node distance=0cm, xshift=1.75cm, yshift=-1cm, font=\color{black}] {Distance to optimum};
                        \node[below=of img3, node distance=0cm, xshift=.15cm, yshift=1.25cm,font=\color{black}] {Iterations};
                    \end{tikzpicture}
                    \vspace{-0.02\textheight}
                    \caption{
                        Complex momentum helps correct rotational dynamics when training a Dirac-GAN~\citep{mescheder2018training}.
                        \emph{Left:} Parameter trajectories with step size $\lr \!=\! 0.1$ and momentum $\momCoeff \!=\! 0.9 \exp(i \nicefrac{\pi}{8})$.
                        We include the classical, real and positive momentum which diverges for any step size.
                        \emph{Right:} The distance from optimum, which has a linear convergence rate matching our prediction with Theorem~\ref{thm:theorem} and (\ref{eq:spectrum_desired}).
                    }
                    \label{fig:trajectories_limited}
                \vspace{-0.01\textheight}
            \end{figure}
            %
            %
            
            %
        \vspace{-0.0125\textheight}
        \section{Complex Momentum}\label{sec:theory}
        \vspace{-0.0125\textheight}
            We describe our proposed method, where the momentum coefficient $\momCoeff \in \mathbb{C}$, step size $\lr \in \mathbb{R}$, momentum buffer $\momVal \in \mathbb{C}^{\bothDim}$, and player parameters $\bothParam \in \mathbb{R}^{\bothDim}$.
            The simultaneous (or Jacobi) update is:
            \vspace{-0.005\textheight}
            \newcommand{\history}{\boldsymbol{h}}
            \begin{align*}\label{eq:simul_update}\tag{SimCM}
                \smash{
                \momVal^{j\!+\!1} = \momCoeff \momVal^{j} - \bothGrad^j,\eqspace
                \bothParam^{j\!+\!1} = \bothParam^j + \Re(\lr \momVal^{j\!+\!1})
                }
            \end{align*}
            There are many ways to get a real-valued update from $\momVal \in \mathbb{C}$, but we only consider updates equivalent to classical momentum when $\momCoeff \in \mathbb{R}$.
            Specifically, we simply update the parameters using the real component of the momentum: $\Re(\momVal)$.
            
            %
            \begin{wrapfigure}{r}{0.35\linewidth}
                \vspace{-0.04\textheight}
                \begin{minipage}{0.35\textwidth}
                    \begin{algorithm}[H]
                        \caption{(\ref{eq:simul_update}) Momentum} \label{alg:simultaneous_complex}
                        \begin{algorithmic}[1]
                            \State $\smash{\momCoeff, \lr \in \mathbb{C}, \momVal \in \mathbb{C}^{\bothDim}, \bothParam^{\num{0}} \in \mathbb{R}^{\bothDim}}$
                            \For{$i = \num{1} \dots N$ $\vphantom{\outGrad^j}$}
                                \State $\smash{\momVal^{j\!+\!1} \!=\! \momCoeff \momVal^j \!-\! \bothGrad^j}$ 
                                \State $\bothParam^{j\!+\!1} \!=\! \bothParam^{j} \!+\! \Re(\lr \momVal^{j\!+\!1})$ 
                            \EndFor
                            \Return $\bothParam^N$
                        \end{algorithmic}
                    \end{algorithm}
                \end{minipage}
                \vspace{-0.04\textheight}
            \end{wrapfigure}
            We show the \ref{eq:simul_update} update in Algorithm~\ref{alg:simultaneous_complex} and visualize it in Figure~\ref{fig:cm_comp_graph}.
            We also show the alternating (or Gauss-Seidel) update, which is common for GAN training:
            %
            \begin{align*}\label{eq:alt_update}\tag{AltCM}
                \momVal^{j\!+\!1}_{\outSymbol} \!&=\! \momCoeff \momVal^{j}_{\outSymbol} \!-\! \bothGrad_{\outSymbol}(\bothParam^j\!),
                \outParam^{j\!+\!1} \!=\! \outParam^j \!+\! \Re(\lr \momVal^{j\!+\!1}_{\outSymbol}\!)\\
                \momVal^{j\!+\!1}_{\inSymbol} \!&=\! \momCoeff \momVal^{j}_{\inSymbol} \!-\! \bothGrad_{\inSymbol}(\outParam^{j\!+\!1}\!\!, \inParam^{j}\!),
                \inParam^{j\!+\!1} \!=\! \inParam^j \!+\! \Re(\lr \momVal^{j\!+\!1}_{\inSymbol}\!)
            \end{align*}
            %
            
            \paragraph{Generalizing negative momentum:} Consider the negative momentum from \citet{gidel2018negative}: $\bothParam^{j\!+\!1} = \bothParam^j - \lr \bothGrad^j + \momCoeff (\bothParam^j - \bothParam^{j-1})$.
            %
            Expanding (\ref{eq:simul_update}) with $\momVal^j = \nicefrac{(\bothParam^j - \bothParam^{j-1})}{\lr}$ for real momentum shows the negative momentum method of \citet{gidel2018negative} is a special case of our method:
            \begin{align}
                \bothParam^{j\!+\!1} = \bothParam^j + \Re(\lr (\momCoeff\nicefrac{(\bothParam^j - \bothParam^{j-1})}{\lr} - \bothGrad^j))
                = \bothParam^j - \lr \bothGrad^j + \momCoeff (\bothParam^j - \bothParam^{j-1})
            \end{align}
            %
            
            \vspace{-0.0125\textheight}
            \subsection{Dynamics of Complex Momentum}\label{sec:dynamics}
            \vspace{-0.0125\textheight}
                For simplicity, we assume Numpy-style~\citep{numpy} component-wise broadcasting for operations like taking the real-part $\Re(\boldsymbol{z})$ of vector $\boldsymbol{z} = [z_1, \dots, z_n] \in \mathbb{C}^{n}$, with proofs in the Appendix.
                
                \new{
                Expanding the buffer updates with the polar components of $\momCoeff$ gives intuition for complex momentum:
                \begin{align}\label{eq:polar_complex_update}
                    \begin{split}
                        \vphantom{A^{A^{A^{A}}}}
                        \momVal^{j\!+\!1} \!=\! \momCoeff \momVal^{j} - \bothGrad^j &\iff 
                        \momVal^{j\!+\!1} \!=\! \momCoeff (\momCoeff (\cdots) - \bothGrad^{j-1}) - \bothGrad^j \iff
                        \momVal^{j\!+\!1} \!=\! \smash{-\sum_{k=0}^{k=j}\momCoeff^{k}\bothGrad^{j-k}} \iff\\
                        \Re(\momVal^{j\!+\!1}) \!=\! -&\sum_{k=0}^{k=j}|\momCoeff|^{k} \cos(k \arg(\momCoeff))\bothGrad^{j-k},\eqspace
                        \Im(\momVal^{j\!+\!1}) \!=\! -\sum_{k=0}^{k=j}|\momCoeff|^{k} \sin(k \arg(\momCoeff))\bothGrad^{j-k} 
                    \end{split}
                \end{align}
                The final line is simply by Euler's formula (\ref{eq:euler_formula}).
                From (\ref{eq:polar_complex_update}) we can see $\momCoeff$ controls the momentum buffer $\momVal$ by having $|\momCoeff|$ dictate prior gradient decay rates, while $\arg(\momCoeff)$ controls oscillation frequency between adding and subtracting prior gradients, which we visualize in Figure~\ref{fig:frequency_vis}.
                }
                \begin{figure}
                    \vspace{-0.065\textheight}
                    \centering
                    \begin{subfigure}{.43\textwidth}
                        \begin{tikzpicture}
                            \centering
                            \node (img1){\includegraphics[trim={1.0cm .6cm .97cm .9cm},clip,width=.89\linewidth]{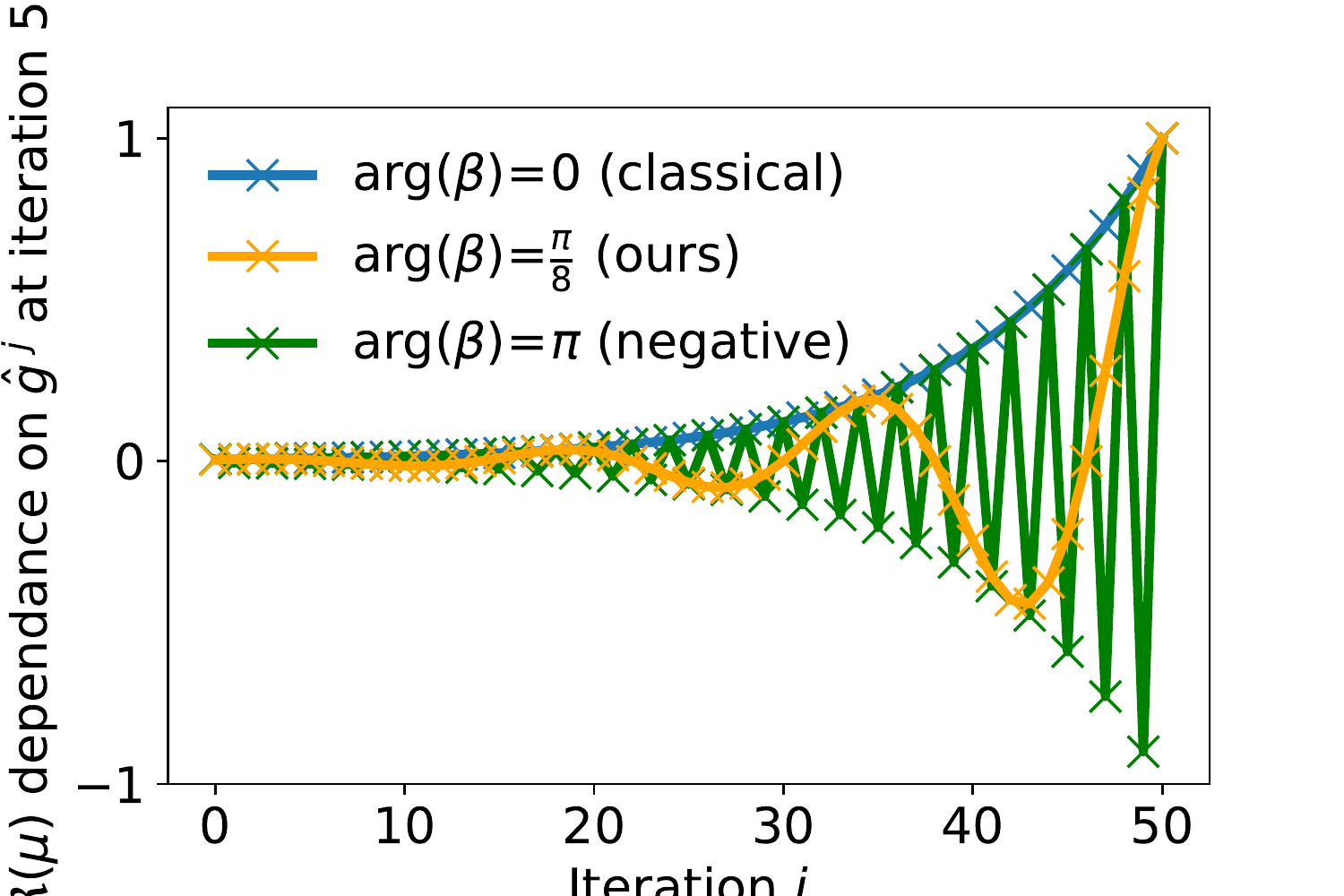}};
                            \node[left=of img1, node distance=0cm, rotate=90, xshift=2.0cm, yshift=-.8cm, font=\color{black}] {Dependence on gradient $k$};
                            \node[below=of img1, node distance=0cm, xshift=-.1cm, yshift=1.2cm,font=\color{black}] {Iteration $k$};
                        \end{tikzpicture}
                        \vspace{-0.02\textheight}
                        \newsubcap{
                            \new{
                            We show the real part of our momentum buffer -- which dictates the parameter update -- at the $50^{th}$ iteration $\Re(\momVal^{50})$ dependence on past gradients $\bothGrad^{k}$ for $k \!=\! 1 \dots 50$.
                            The momentum magnitude is fixed to $| \momCoeff| \!=\! \num{.9}$ as in Figure~\ref{fig:trajectories_limited}.
                            Euler's formula is used in (\ref{eq:polar_complex_update}) to for finding dependence or coefficient of $\bothGrad^{k}$ via $\Re(\momVal^{50}) \!=\! -\sum_{k=0}^{k=50}|\momCoeff|^{k} \cos(k \arg(\momCoeff))\bothGrad^{j-k}$.
                            Complex momentum allows smooth changes in the buffers dependence on past gradients.
                            }
                        }
                        \label{fig:frequency_vis}
                        \end{subfigure}
                    \hspace{0.01\textwidth}
                    \begin{subfigure}{.54\textwidth}
                        \begin{tikzpicture}
                            \centering
                            \node (img1){\includegraphics[trim={1.0cm .9cm .97cm .9cm},clip,width=.87\linewidth]{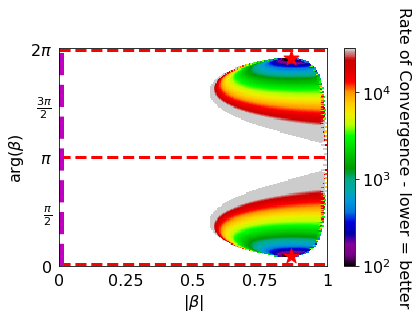}};
                            \node[left=of img1, node distance=0cm, rotate=90, xshift=2.0cm, yshift=-.9cm, font=\color{black}] {Momentum phase $\arg(\momCoeff)$};
                            \node[below=of img1, node distance=0cm, xshift=-.1cm, yshift=1.3cm,font=\color{black}] {Momentum magnitude $| \momCoeff |$};
                            \node[right=of img1, node distance=0cm, rotate=270, xshift=-2.1cm, yshift=-.9cm, font=\color{black}] {Number of steps to converge};
                        \end{tikzpicture}
                        \vspace{-0.03\textheight}
                        \newsubcap{
                            How many steps simultaneous complex momentum on a Dirac-GAN takes for a set solution distance.
                            We fix step size $\lr \!=\! \num{0.1}$ as in Figure~\ref{fig:trajectories_limited}, while varying the phase and magnitude of our momentum $\momCoeff  \!=\! |\momCoeff| \exp(i \arg(\momCoeff))$.
                            There is a {\color{red}red} star at the optima, dashed {\color{red}red} lines at real $\momCoeff$, and a dashed {\color{magenta}magenta} line for simultaneous gradient descent.
                            There are no real-valued $\momCoeff$ that converge for this -- or any -- $\lr$ with simultaneous updates ~\citep{gidel2018negative}.
                            Appendix Figure~\ref{fig:det_minmax_expAvg_alt_phase_tune} compares this with alternating updates (\ref{eq:alt_update}).
                        }
                        \label{fig:det_minmax_expAvg_simul_phase_tune}
                    \end{subfigure}
                    \vspace{-0.025\textheight}
                \end{figure}
                
                Expanding the parameter updates with the Cartesian components of $\lr$ and $\momCoeff$ is key for Theorem~\ref{thm:theorem}, which characterizes the convergence rate:
                \vspace{-0.005\textheight}
                \begin{align}\label{eq:cartesian_complex_update}
                    \begin{split}
                        \smash{\momVal^{j\!+\!1}} \!=\! \smash{\momCoeff \momVal^{j} - \bothGrad^j \iff} \\
                        \smash{\Re(\momVal^{j\!+\!1})} \!=\! \smash{\Re(\momCoeff) \! \Re(\momVal^j) \!-\! \Im(\momCoeff) \! \Im(\momVal^j)  \!-\! \Re(\bothGrad^j)},\,\,\,
                        &\Im(\momVal^{j\!+\!1}) \!=\! \Im(\momCoeff) \! \Re(\momVal^{j}) \!+\! \Re(\momCoeff) \! \Im(\momVal^{j}) 
                    \end{split}
                \end{align}
                \vspace{-0.02\textheight}
                \begin{align}\label{eq:parameter_dynamics}
                        \bothParam^{j\!+\!1} \!=\! \bothParam^j \!+\! \Re(\lr \momVal^{j\!+\!1}) \iff
                        \bothParam^{j\!+\!1} \!=\! \bothParam^j \!-\! \lr\bothGrad^j \!+\! \Re(\lr \momCoeff) \! \Re(\momVal^j) \!-\! \Im(\lr \momCoeff)\!\Im(\momVal^j)
                \end{align}
                So, we can write the next iterate with a fixed-point operator:\vspace{-0.005\textheight}
                \begin{equation}\label{eq:fixed_point_operator}
                    \smash{
            	    [\Re(\momVal^{j\!+\!1}),\! \Im(\momVal^{j\!+\!1}),\! \bothParam^{j\!+\!1}] \!\!=\! \fixedPointOp ([\Re(\momVal^{j}),\! \Im(\momVal^{j}),\! \bothParam^{j}]\!)
            	    }
            	\end{equation}
                (\ref{eq:cartesian_complex_update}) and (\ref{eq:parameter_dynamics}) allow us to write the Jacobian of $\fixedPointOp$ which can be used to bound convergence rates near fixed points, which we name the Jacobian of the augmented dynamics of buffer $\momVal$ and joint-parameters $\bothParam$ and denote with:
                \vspace{-0.01\textheight}
                \begin{equation}
            	    \dynamics \!\defeq\! 
            	    \nabla_{[\momVal, \bothParam]}\fixedPointOp = \begin{bmatrix}
                        \Re(\momCoeff) \identity & -\Im(\momCoeff) \identity & -\nabla_{\bothParam} \bothGrad\\
                        \Im(\momCoeff) \identity & \Re(\momCoeff) \identity & 0\\
                        \Re(\lr \momCoeff) \identity & -\Im(\lr \momCoeff) \identity & \identity \!-\! \lr\nabla_{\bothParam}\bothGrad\\
                        \end{bmatrix}
            	\end{equation}
            	So, for quadratic losses our parameters evolve via:
            	\begin{equation}\label{eq:dyn_system}
            	    \smash{
            	    [\Re(\momVal^{j\!+\!1}),\! \Im(\momVal^{j\!+\!1}),\! \bothParam^{j\!+\!1}]^{\transpose} \!\!=\! \dynamics \, [\Re(\momVal^{j}),\! \Im(\momVal^{j}),\! \bothParam^{j}]^{\transpose}\!
            	    }
            	\end{equation}
            	We can bound convergence rates by looking at the spectrum of $\dynamics$ with Theorem~\ref{thm:theorem}.
                \begin{restatable}[Consequence of Prop. 4.4.1 ~\citet{bertsekas2008nonlinear}]{thm}{mainThm}
                    \label{thm:theorem}
                    Convergence rate of complex momentum:
                    If the spectral radius $\spectralRadius(\dynamics) \!=\! \spectralRadius(\nabla_{[\momVal, \bothParam]}\fixedPointOp) \!<\! 1$, then, for $[\momVal, \bothParam]$ in a neighborhood of $[\momVal^*\!, \bothParam^*]$, the distance of $[\momVal^{j}\!, \bothParam^{j}]$ to the stationary point $[\momVal^*\!, \bothParam^*]$ converges at a linear rate $\mathcal{O}((\spectralRadius(\dynamics) + \epsilon)^j), \forall \epsilon \!>\! 0$.
                \end{restatable}
                %
                %
                \vspace{-0.005\textheight}
                Here, linear convergence means $\lim_{j \to \infty} \!\! \nicefrac{\| \bothParam^{j\!+\!1} - \bothParam^*\|}{\| \bothParam^{j} - \bothParam^*\|} \!\in\! (\num{0}, \num{1})$, where $\bothParam^*$ is a fixed point.
                We should select optimization parameters $\lr, \momCoeff$ so that the augmented dynamics spectral radius $\spectrum(\dynamics(\lr, \momCoeff)) \!<\! 1$---with the dependence on $\lr$ and $\momCoeff$ now explicit.
                We may want to express $\spectrum(\dynamics(\lr, \momCoeff))$ in terms of the spectrum $\spectrum(\nabla_{\bothParam}\bothGrad)$, as in Theorem 3 in \citet{gidel2018negative}:
                \vspace{-0.005\textheight}
                \begin{equation}\label{eq:spectrum_desired}
                    \smash{\boldsymbol{f}(\spectrum(\nabla_{\bothParam}\bothGrad), \lr, \momCoeff) \!=\! \spectrum(\dynamics(\lr, \momCoeff))}
                \end{equation}
                We provide a Mathematica command in Appendix~\ref{sec:poly} for a cubic polynomial $p$ characterizing $\boldsymbol{f}$ with coefficients that are functions of $\lr, \momCoeff$ \& $\eigval \in \spectrum(\nabla_{\bothParam}\bothGrad)$, whose roots are eigenvalues of $\dynamics$, which we use in subsequent results.
                \citet{o2015adaptive, lucas2018aggregated} mention that in practice we do not know the condition number, eigenvalues -- or the mixture of cooperative and adversarial eigenspaces  -- of a set of functions that we are optimizing, so we try to design algorithms which work over a large range.
                Sharing this motivation, we consider convergence behavior on games ranging from purely adversarial to cooperative.
                
                In Section~\ref{sec:exp_spectrum} at every non-real $\momCoeff$ we could select $\lr$ and $| \momCoeff |$ so Algorithm~\ref{alg:simultaneous_complex} converges.
                We define \emph{almost-positive} to mean $\arg(\momCoeff) \!=\! \epsilon$ for small $\epsilon$, and show there are almost-positive $\momCoeff$ which converge.\vspace{-0.01\textheight}
                \begin{restatable}[Convergence of Complex Momentum]{result}{resultExist}
                    \label{thm:theorem_existence}
                         There exist $\lr \in \mathbb{R}, \momCoeff \in \mathbb{C}$ so Algorithm~\ref{alg:simultaneous_complex} converges for bilinear zero-sum games.
                         More-so, for small $\epsilon$ (we show for $\epsilon = \frac{\pi}{\num{16}}$), if $\arg(\momCoeff) = \epsilon$ (i.e., almost-positive) or $\arg(\momCoeff) = \pi - \epsilon$ (i.e., almost-negative), then we can select $\lr, |\momCoeff|$ to converge.
                \end{restatable}
                \vspace{-0.0075\textheight}
                \textbf{Why show this?}
                Our result complements \citet{gidel2018negative} who show that for all real $\lr, \momCoeff$ Algorithm~\ref{alg:simultaneous_complex} \emph{does not} converge.
                We include the proof for bilinear zero-sum games, but the result generalizes to some games that are purely adversarial near fixed points, like Dirac GANs~\citep{mescheder2017numerics}.
                The result's second part shows evidence there is a sense in which the only $\momCoeff$ that do not converge are real (with simultaneous updates on purely adversarial games).
                It also suggests a form of robustness, because almost-positive $\momCoeff$ can approach acceleration in cooperative eigenspaces, while converging in adversarial eigenspaces, so almost-positive $\momCoeff$ may be desirable when we have games with an uncertain or variable mixtures of real and imaginary eigenvalues like GANs.
                Sections~~\ref{sec:exp_spectrum}, \ref{sec:train_small_gan}, and \ref{sec:train_big_gan} investigate this further.
                
            \vspace{-0.0125\textheight}
            \subsection{What about Acceleration?}\label{sec:acceleration}
            \vspace{-0.0125\textheight}
                With classical momentum, finding the step size $\lr$ and momentum $\momCoeff$ to optimize the convergence rate tractable if $\num{0} \!<\! l \!\leq\! L$ and $\smash{\spectrum(\nabla_{\bothParam}\bothGrad) \!\in\! [l, L]^{\bothDim}}$~\citep{goh2017momentum} -- i.e., we have an $l$-strongly convex and $L$-Lipschitz loss.
                The conditioning $\condition \!=\! \nicefrac{L}{l}$ can characterize problem difficulty.
                Gradient descent with an appropriate $\lr$ can achieve a convergence rate of $\smash{\frac{\condition - \num{1}}{\condition + \num{1}}}$, but using momentum with appropriate $(\lr^*\!, \momCoeff^*)$ can achieve an \emph{accelerated} rate of $\smash{\spectralRadius^* \!=\! \frac{\sqrt{\condition} - \num{1}}{\sqrt{\condition} + \num{1}}}$.
                However, there is no consensus for constraining $\smash{\spectrum(\nabla_{\bothParam}\bothGrad)}$ in games for tractable and useful results.
                Candidate constraints include monotonic vector fields generalizing notions of convexity, or vector fields with bounded eigenvalue norms capturing a kind of sensitivity~\citep{azizian2019tight}.
                %
                Figure~\ref{fig:gan_spectrum_main} shows $\smash{\spectrum(\nabla_{\bothParam}\bothGrad)}$ for a GAN -- we can attribute some eigenvectors to a single player's parameters.
                The discriminator can be responsible for the largest and smallest norm eigenvalues, suggesting we may benefit from varying $\lr$ and $\momCoeff$ for each player as done in Section~\ref{sec:train_big_gan}.
            \vspace{-0.0125\textheight}
            \subsection{Implementing Complex Momentum}
            \vspace{-0.0125\textheight}
                Complex momentum is trivial to implement with libraries supporting complex arithmetic like JAX~\citep{jax2018github} or Pytorch~\citep{paszke2017automatic}.
                Given an SGD implementation, we often only need to change a few lines of code -- see Figure~\ref{fig:jax_code_change}.
                Also, (\ref{eq:cartesian_complex_update}) and (\ref{eq:parameter_dynamics}) can be easily used to implement Algorithm~\ref{alg:simultaneous_complex} in a library without complex arithmetic.
                More sophisticated optimizers like Adam can trivially support complex optimizer parameters with real-valued updates, which we explore in Section~\ref{sec:train_big_gan}.
                
            \new{
            \vspace{-0.0125\textheight}
            \subsection{Scope and Limitations}
            \vspace{-0.0125\textheight}
                For some games, we need higher than first-order information to converge -- ex., pure-response games ~\citep{lorraine2019optimizing} -- because the first-order information for a player is identically zero.
                So, momentum methods only using first-order info will not converge in general.
                However, we can combine methods with second-order information and momentum algorithms~\citep{lorraine2019optimizing, raghu2020teaching}.
                Complex momentum's computational cost is almost identical to classical and negative momentum, except we now have a buffer with twice as many real parameters.
                We require one more optimization hyperparameter than classical momentum, which we provide an initial guess for in Section~\ref{sec:init_guess}.
            }
    \newcommand{\fixedPhase}{\frac{\pi}{\num{16}}}
    \vspace{-0.0175\textheight}
    \section{Experiments}\label{sec:experiments}
    \vspace{-0.015\textheight}
        We investigate complex momentum's performance in training GANs and games with different mixtures of cooperative and adversarial eigenspaces, showing improvements over standard baselines.
        Code for experiments will be available on publication, with reproducibility details in Appendix~\ref{app:experiments}.
        
        \textbf{Overview:} We start with a purely adversarial Dirac-GAN and zero-sum games, which have known solutions $\bothParam^* \!=\! (\outParam^*, \inParam^*)$ and spectrums $\spectrum(\nabla_{\bothParam} \bothGrad)$, so we can assess convergence rates.
        Next, we evaluate GANs generating $\num{2}$D distributions, because they are simple enough to train with a plain, alternating SGD.
        Finally, we look at scaling to larger-scale GANs on images which have brittle optimization, and require optimizers like Adam.
        Complex momentum provides benefits in each setup.
        
        \new{
        We only compare to first-order optimization methods, despite there being various second-order methods due limitations discussed in Section~\ref{sec:limitation_existing}.
        }
        \vspace{-0.0125\textheight}
        \subsection{Optimization in Purely Adversarial Games}\label{sec:single_param_det}
        \vspace{-0.0125\textheight}
            \newcommand{\bilinearInit}{(\num{1}, \num{1})}
            \newcommand{\bilinearLrMag}{\num{2}}
            \newcommand{\bilinearMomMag}{\nicefrac{\num{1}}{\num{3}}}
            \newcommand{\bilinearIterations}{\num{300}}
            \newcommand{\bilinearSmoothIterations}{\num{50}}
            \newcommand{\bilinearPhase}{\fixedPhase}
                \new{
                Here, we consider the optimizing the Dirac-GAN objective, which is surprisingly hard and where many classical optimization methods fail, because $\spectrum(\nabla_{\bothParam} \bothGrad)$ is imaginary near solutions:
                \vspace{-0.005\textheight}
                \begin{equation}\label{eq:min_max_xy}
                    \smash{
                    \min_x \max_y -\log(1 + \exp(-xy)) - \log(2)
                    }
                \end{equation}
                }
                Figure~\ref{fig:trajectories_limited} empirically verifies convergence rates given by Theorem~\ref{thm:theorem} with (\ref{eq:spectrum_desired}), by showing the optimization trajectories with simultaneous updates.
                
                Figure~\ref{fig:det_minmax_expAvg_simul_phase_tune} investigates how the components of the momentum $\momCoeff$ affect convergence rates with simultaneous updates and a fixed step size.
                The best $\momCoeff$ was almost-positive (i.e., $\arg(\momCoeff) \!=\! \epsilon$ for small $\epsilon$).
                We repeat this experiment with alternating updates in Appendix Figure~\ref{fig:det_minmax_expAvg_alt_phase_tune}, which are standard in GAN training.
                There, almost-positive momentum is best (but negative momentum also converges), and the benefit of alternating updates can depend on if we can parallelize player gradient evaluations.
        
        \newcommand{\ganSize}{\num{16}}
        \vspace{-0.02\textheight}
        \new{
        \subsection{How Adversarialness Affects Convergence Rates}\label{sec:exp_spectrum}
        \vspace{-0.01\textheight}
            Here, we compare optimization with first-order methods for purely adversarial, cooperative, and mixed games.
            We use the following game, allowing us to easily interpolate between these regimes:\vspace{-0.005\textheight}
            }
            \begin{equation}\label{eq:bilinear_adversarial_interpolate}
                \smash{
                    \min_{\boldsymbol{x}} \max_{\boldsymbol{y}}  \boldsymbol{x}^\transpose (\boldsymbol{\gamma} \boldsymbol{A}) \boldsymbol{y} + \boldsymbol{x}^\transpose ((\identity - \boldsymbol{\gamma})\boldsymbol{B}_1) \boldsymbol{x} - \boldsymbol{y}^\transpose ((\identity - \boldsymbol{\gamma})\boldsymbol{B}_2) \boldsymbol{y}
                }
            \end{equation}
             \begin{wrapfigure}[34]{r}{.51\textwidth}
                \vspace{-0.035\textheight}
                \centering
                \begin{tikzpicture}
                    \centering
                    \node (img1){\includegraphics[trim={.0cm .0cm .0cm .0cm},clip,width=.93\linewidth]{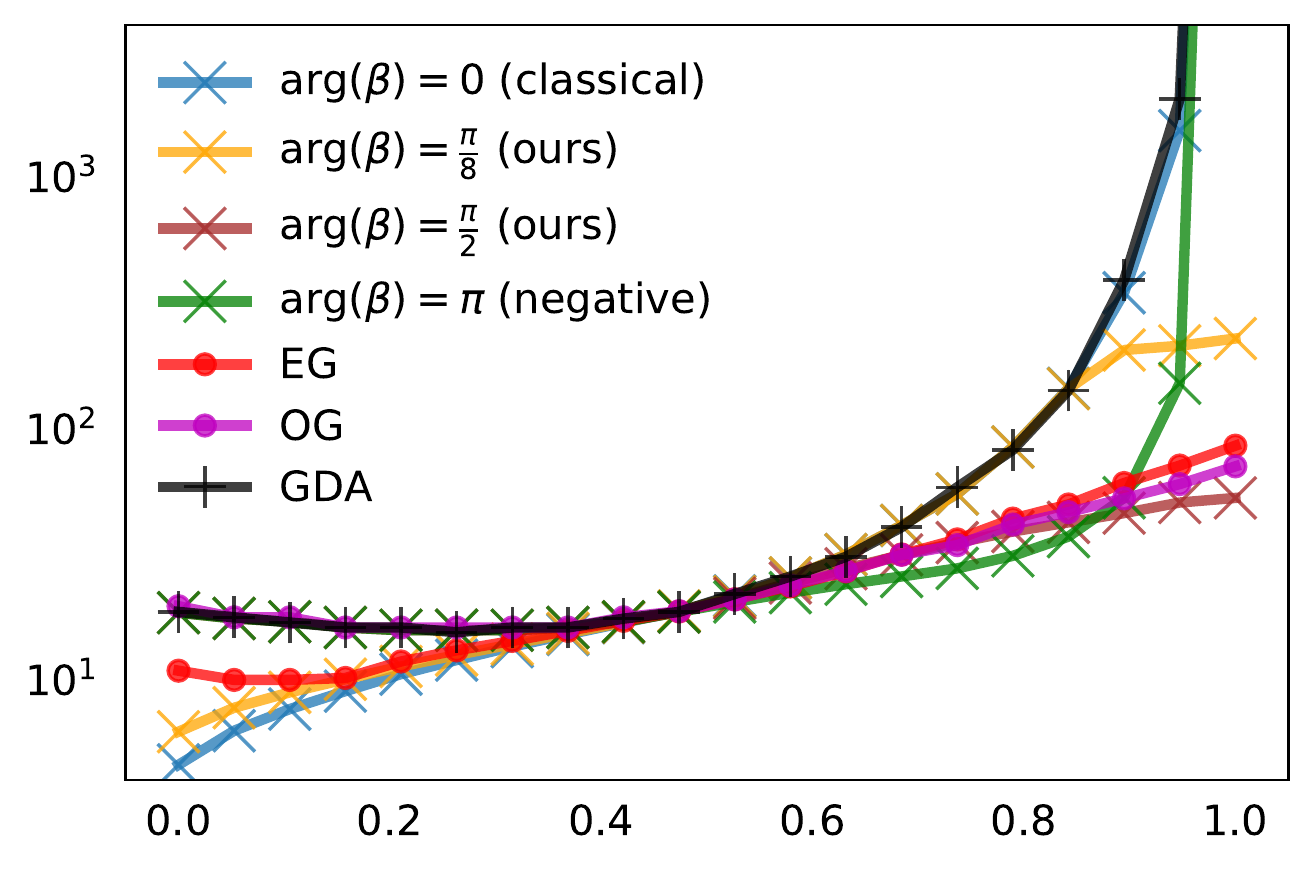}};
                    \node[left=of img1, node distance=0cm, rotate=90, xshift=2.0cm, yshift=-1.0cm, font=\color{black}] {\# grad. eval. to converge};
                    \node[below=of img1, node distance=0cm, xshift=-.1cm, yshift=1.3cm,font=\color{black}]{Max adversarialness $\gamma_{max}$};
                \end{tikzpicture}
                \vspace{-0.025\textheight}
                \caption{
                    \new{
                    We compare first-order methods convergence rates on the game in (\ref{eq:bilinear_adversarial_interpolate}), with $\boldsymbol{A} \!=\! \boldsymbol{B}_1 \!=\! \boldsymbol{B}_2$ diagonal and entries linearly spaced in $[\nicefrac{\num{1}}{\num{4}}, \num{4}]$.
                    We interpolate from purely cooperative to a mixture of purely cooperative and adversarial eigenspaces in $\spectrum(\nabla_{\bothParam}\bothGrad)$ by making $\boldsymbol{\gamma}$ diagonal with $\gamma_j \!\sim\! U[0, \gamma_{max}]$, inducing $j^{th}$ eigenvalue pair to have $\arg(\eigval_j) \!\approx\! \pm \gamma_j \frac{\pi}{2}$.
                    So, $\gamma_{max}$ controls the largest possible eigenvalue $\arg$ or \emph{max adversarialness}.
                    Every method generalizes gradient descent-ascent (GDA) by adding an optimizer parameter, tuned via grid search.
                    {\color{blue}Positive momentum} and {\color{green}negative momentum} do not converge if there are purely adversarial eigenspaces (i.e., $\gamma_{max} \!=\! 1$).
                    {\color{orange}Almost-positive momentum $\arg(\momCoeff) \!=\! \epsilon$} $>\! 0$ like ${\color{orange}\nicefrac{\pi}{8}}$ allows us to approach the acceleration of positive momentum if sufficiently cooperative (i.e., $\gamma_{max} \!<\! \num{.5}$), while still converging if there are purely adversarial eigenspaces (i.e., $\gamma_{max} \!=\! 1$).
                    Tuning $\arg(\momCoeff)$ with complex momentum performs competitively with {\color{red}extragradient (EG)}, {\color{magenta}optimistic gradient (OG)} for any adversarialness -- ex., ${\color{brown}\arg(\momCoeff) \!=\! \nicefrac{\pi}{2}}$ does well if there are purely adversarial eigenspaces (i.e., $\gamma_{max} \!=\! 1$).
                    }
                }
                \vspace{-0.00\textheight}
                \label{fig:partial_coop_phases}
            \end{wrapfigure}
            If $\boldsymbol{\gamma} \!=\! \identity$ the game is purely adversarial, while if the $\boldsymbol{\gamma} \!=\! \boldsymbol{0}$ the game is purely cooperative.
            
            Figure~\ref{fig:spectrum_vary_phase} explores $\spectrum(\dynamics)$ in purely adversarial games for a range of $\lr, \momCoeff$, generalizing Figure 4 in \citet{gidel2018negative}.
            At every non-real $\momCoeff$---i.e., $\arg(\momCoeff) \!\neq\! \pi$ or $\num{0}$---we could select $\lr, |\momCoeff|$ that converge.
            \newcommand{\spectrumVaryWidth}{0.19\linewidth} 
            \newcommand{\spectrumAxisWidth}{0.012\linewidth}
            \newcommand{\spectrumCbarWidth}{0.03\linewidth}
            \newcommand{\spectrumSpacer}{-1.45cm}
            \begin{figure}[t]
                \vspace{-0.03\textheight}
                \centering
                \begin{tikzpicture}
                    \centering
                    \node (img0){\includegraphics[trim={1.2cm 1.05cm 23.5cm .7cm},clip,width=\spectrumAxisWidth]{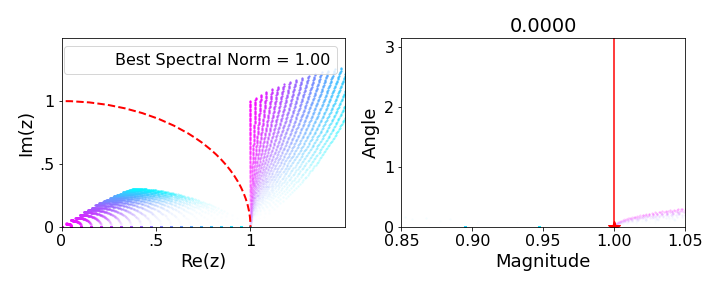}};
                    \node[left=of img1, node distance=0cm, rotate=90, xshift=1.0cm, yshift=-4cm, font=\color{black}] {$\Im(\spectrum(\dynamics))$};
                    
                    \node (img1)[right=of img0, node distance=0cm, xshift=-1.265cm]{\includegraphics[trim={1.9cm 1.25cm 12.5cm .7cm},clip,width=\spectrumVaryWidth]{images/spectrum/phase=0.png}};
                    \node[above=of img1, node distance=0cm, xshift=-0cm, yshift=-1.25cm,font=\color{black}] {$\arg(\momCoeff) = \num{0}$};
                    
                    \node (img2)[right=of img1, node distance=0cm, xshift=\spectrumSpacer]{\includegraphics[trim={1.9cm 1.25cm 12.5cm .7cm},clip,width=\spectrumVaryWidth]{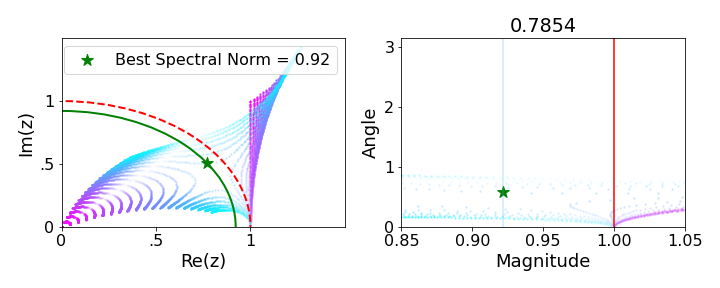}};
                    \node[above=of img2, node distance=0cm, xshift=0cm, yshift=-1.25cm,font=\color{black}] {$\arg(\momCoeff) = \frac{\pi}{\num{4}}$};
                     
                    \node (img3)[right=of img2, node distance=0cm, xshift=\spectrumSpacer]{\includegraphics[trim={1.9cm 1.25cm 12.5cm .7cm},clip,width=\spectrumVaryWidth]{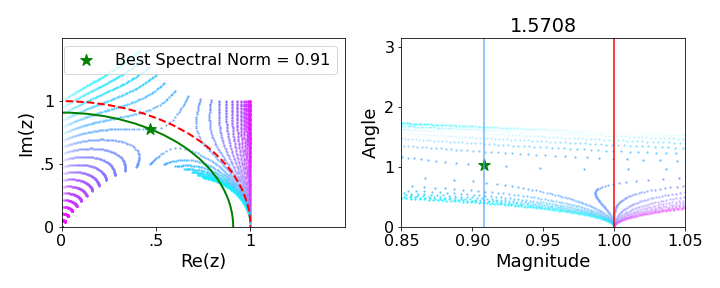}};
                    \node[below=of img3, node distance=0cm, xshift=-.1cm, yshift=1.2cm,font=\color{black}] {The real component of the spectrum of the augmented learning dynamics $\Re(\spectrum(\dynamics))$};
                    \node[above=of img3, node distance=0cm, xshift=0cm, yshift=-1.25cm,font=\color{black}] {$\arg(\momCoeff) = \frac{\pi}{\num{2}}$};
                     
                    \node (img4)[right=of img3, node distance=0cm, xshift=\spectrumSpacer]{\includegraphics[trim={1.9cm 1.25cm 12.5cm .7cm},clip,width=\spectrumVaryWidth]{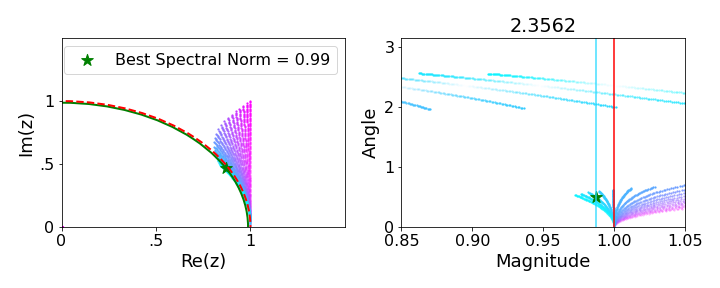}};
                    \node[above=of img4, node distance=0cm, xshift=0cm, yshift=-1.25cm,font=\color{black}] {$\arg(\momCoeff) =\ \frac{\num{3}\pi}{\num{4}}$};
                     
                    \node (img5)[right=of img4, node distance=0cm, xshift=\spectrumSpacer]{\includegraphics[trim={1.9cm 1.25cm 12.5cm .7cm},clip,width=\spectrumVaryWidth]{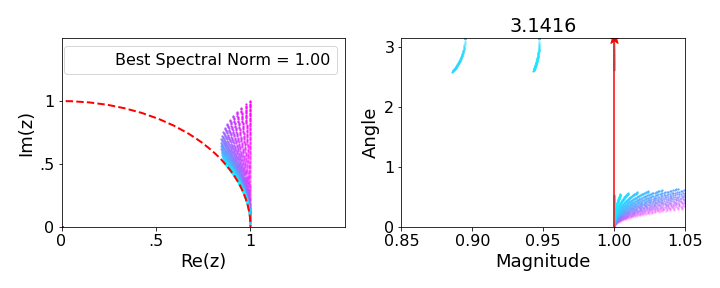}};
                    \node[above=of img5, node distance=0cm, xshift=0cm, yshift=-1.25cm,font=\color{black}] {$\arg(\momCoeff) = \pi$};
                    
                    \node (img6)[right=of img5, node distance=0cm, xshift=\spectrumSpacer]{\includegraphics[trim={.55cm .75cm .75cm .55cm},clip,width=\spectrumCbarWidth]{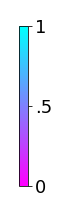}};
                    \node[right=of img6, node distance=0cm, rotate=270, xshift=0.0cm, yshift=-1.cm, font=\color{black}] {$| \momCoeff |$};
                \end{tikzpicture}
                \vspace{-0.025\textheight}
                \caption{
                    The spectrum of the augmented learning dynamics $\dynamics$ is shown, whose spectral norm is the convergence rate in Theorem~\ref{thm:theorem}.
                    Each image is a different momentum phase $\arg(\momCoeff)$ for a range of $\lr,\! |\momCoeff| \!\in\! [\num{0},\! \num{1}]$.
                    The opacity of an eigenvalue (eig) is the step size $\lr$ and the color corresponds to momentum magnitude $|\momCoeff|$.
                    A {\color{red}red} unit circle shows where all eigs must lie to converge for a fixed $\lr, \momCoeff$.
                    If the max eig norm $<\! \num{1}$, we draw a {\color{green}green} circle whose radius is our convergence rate and a {\color{green}green} star at the associated eig.
                    Notably, at every non-real $\momCoeff$ we can select $\lr,\! |\momCoeff|$ for convergence.
                    The eigs are symmetric over the $x$-axis, and eigs near $\Re(\eigval) \!=\! \num{1}$ dictate convergence rate.
                    Eigs near the center are due to state augmentation, have small magnitudes, and do not impact convergence rate.
                    Simultaneous gradient descent corresponds to the {\color{magenta}magenta} values where $|\momCoeff| \!=\! \num{0}$.
                }
                \label{fig:spectrum_vary_phase}
                \vspace{-0.0225\textheight}
            \end{figure}
            
            %
            Figure~\ref{fig:partial_coop_phases} compares first-order algorithms as we interpolate from the purely cooperative games (i.e., minimization) to mixtures of purely adversarial and cooperative eigenspaces, because this setup range can occur during GAN training -- see Figure~\ref{fig:gan_spectrum_main}.
            Our baselines are simultaneous SGD (or gradient descent-ascent (GDA)), extragradient (EG) \citep{korpelevich1976extragradient}, optimistic gradient (OG) \citep{chiang2012online, rakhlin2013optimization, daskalakis2017training}, and momentum variants.
            We added extrapolation parameters for EG and OG so they are competitive with momentum -- see Appendix Section~\ref{sec:app_spectrum}.
            We show how many gradient evaluations for a set solution distance, and EG costs two evaluations per update.
            We optimize convergence rates for each game and method by grid search, as is common for optimization parameters in deep learning.
            
            \textbf{Takeaway:} In the cooperative regime -- i.e., $\gamma_{max} \!<\! .5$ or $\max_{\eigval \in \spectrum(\nabla_{\bothParam}\bothGrad)} |\arg(\eigval)| \!<\! \nicefrac{\pi}{4}$ -- the best method is classical, positive momentum, otherwise we benefit from a method for learning in games.
            If we have purely adversarial eigenspaces then GDA, positive and negative momentum fail to converge, while EG, OG, and complex momentum can converge.
            In games like GANs, our eigendecomposition is infeasible to compute and changes during training -- see Appendix Figure~\ref{fig:gan_decomp} -- so we want an optimizer that converges robustly.
            Choosing any non-real momentum $\momCoeff$ allows robust convergence for every eigenspace mixture.
            More so, almost-positive momentum $\momCoeff$ allows us to approach acceleration when cooperative, while still converging if there are purely adversarial eigenspaces.
            %
            %
            %
            \newcommand{\strongConvexity}{\nicefrac{\num{1}}{\num{2}}}
            \newcommand{\lipschitz}{\num{2}}
            \newcommand{\spectrumGraphWidth}{.4\textwidth}
        \newcommand{\optValComplexLowDimGAN}{\num{.76}}
        \newcommand{\optValRealLowDimGAN}{\num{.79}}
        \subsection{Training GANs on $\num{2}$D Distributions}\label{sec:train_small_gan}
        \vspace{-0.0125\textheight}
            \new{
            Here, we investigate improving GAN training using alternating gradient descent updates with complex momentum.
            We look at alternating updates, because they are standard in GAN training~\citep{goodfellow2014generative, brock2018large, wu2019logan}.
            It is not clear how EG and OG generalize to alternating updates, so we use positive and negative momentum as our baselines.
            We train to generate a $\num{2}$D mixture of Gaussians, because more complicated distribution require more complicated optimizers than SGD.
            }
            Figure~\ref{fig:jax_code_change} shows all changes necessary to use the JAX momentum optimizer for our updates, with full details in Appendix~\ref{sec:app_2d_gan}.
            We evaluate the log-likelihood of GAN samples under the mixture as an imperfect proxy for matching.
            
            Appendix Figure~\ref{fig:gan_heatmaps} shows heatmaps for tuning $\arg(\momCoeff)$ and $|\momCoeff|$ with select step sizes.
            \textbf{Takeaway:}
            The best momentum was found at the almost-positive $\momCoeff \approx \num{0.7} \exp(i \nicefrac{\pi}{\num{8}})$ with step size $\lr \!=\! 0.03$, and for each $\lr$ we tested a broad range of non-real $\momCoeff$ outperformed any real $\momCoeff$.
            This suggests we may be able to often improve GAN training with alternating updates and complex momentum.
            
        \vspace{-0.0125\textheight}
        \subsection{Training BigGAN with a Complex Adam}\label{sec:train_big_gan}
        \vspace{-0.0125\textheight}
            \new{
            Here, we investigate improving larger-scale GAN training with complex momentum.
            However, larger-scale GANs train with more complicated optimizers than gradient descent -- like Adam~\citep{kingma2014adam} -- and have notoriously brittle optimization.
            We look at training BigGAN~\citep{brock2018large} on CIFAR-10~\citep{krizhevsky2009learning}, but were unable to succeed with optimizers other than \citep{brock2018large}-supplied setups, due to brittle optimization.
            So, we attempted to change procedure minimally by taking \citep{brock2018large}-supplied code \href{https://github.com/ajbrock/BigGAN-PyTorch}{{\color{blue}here}} which was trained with Adam, and making the $\momCoeff_1$ parameter -- analogous to momentum -- complex.
            The modified complex Adam is shown in Algorithm~\ref{alg:complex_adam}, where the momentum bias correction is removed to better match our theory.
            It is an open question on how to best carry over the design of Adam (or other optimizers) to the complex setting.
            Training each BigGAN took $\num{10}$ hours on an NVIDIA T4 GPU, so Figure~\ref{fig:biggan_inception_heatmap} and Table~\ref{tab:biggan} took about $\num{1000}$ and $\num{600}$ GPU hours respectively.
            }
            
            Figure~\ref{fig:biggan_inception_heatmap} shows a grid search over $\arg(\momCoeff_1)$ and $|\momCoeff_1|$ for a BigGAN trained with Algorithm~\ref{alg:complex_adam}.
            We only changed $\momCoeff_1$ for the discriminator's optimizer.
            \textbf{Takeaway:}
            The best momentum was at the almost-positive $\momCoeff_1 \!\approx\! \num{0.8} \exp(i \nicefrac{\pi}{\num{8}})$, whose samples are in Appendix Figure~\ref{fig:biggan_samples}.
            \begin{figure}[b]
                \vspace{-0.04\textheight}
                \centering
                \begin{minipage}{.31\textwidth}
                    \begin{algorithm}[H]
                        \caption{Complex Adam variant without momentum bias-correction} \label{alg:complex_adam}
                        \begin{algorithmic}[1]
                            \State $\momCoeff_1 \!\in\! \mathbb{C}, \momCoeff_2 \!\in\! [\num{0},\! \num{1})$
                            \State $\lr \!\in\! \mathbb{R}^{+}, \epsilon \!\in\! \mathbb{R}^{+}$
                            \For{$j = 1 \dots N$}
                                \State $\!\!\!\!\!\momVal^{j\!+\!1} \!\!=\! \momCoeff_1 \momVal^{j} - \gradSymbol^j$
                                \State $\!\!\!\!\!\boldsymbol{v}^{j\!+\!1} \!\!=\! \momCoeff_2 \boldsymbol{v}^{j} \!+\! (1\!-\!\momCoeff_2) (\gradSymbol^j)^{\!\num{2}}$
                                \State $\!\!\!\!\!\hat{\boldsymbol{v}}^{j\!+\!1} \!\!=\! \frac{\boldsymbol{v}^{j\!+\!1}}{1 - (\momCoeff_2)^j}$
                                \State $\!\!\!\!\!\bothParam^{j\!+\!1} \!\!=\! \bothParam^{j} + \lr \frac{\Re(\momVal^j)}{\sqrt{\hat{\boldsymbol{v}}^{j\!+\!1}} + \epsilon}$
                            \EndFor
                            \Return $\bothParam^N$
                        \end{algorithmic}
                    \end{algorithm}
                \end{minipage}
                \hspace{0.01\textwidth}
                \begin{minipage}{.65\textwidth}
                    \begin{table}[H]
                        \vspace{-0.00\textheight}
                    	\centering
                    	\begin{tabular}{lll}  
                    		CIFAR-10 BigGAN&\multicolumn{2}{c}{Best IS for \num{10} seeds} \\
                    		\cmidrule(lr){2-3}
                    		Discriminator $\momCoeff_1$ & Min & Max\\  
                    		\midrule
                    		$\num{0}$ -- \citep{brock2018large}'s default & $\num{8.9}$ & $\num{9.1}$ \\
                    		$\num{.8}\exp(i \nicefrac{\pi}{\num{8}})$ -- ours & $\num{8.96} ({\color{green}+.06})$ & $\num{9.25} ({\color{green}+.15})$ \\
                    		$\num{.8}$ & $\num{3.12}({\color{red}-5.78})$ & $\num{9.05}({\color{red}-0.05})$ \\
                    	\end{tabular}
                    	\caption{
                    	    We display the best inception scores (IS) found over $\num{10}$ runs for training BigGAN on CIFAR-10 with various optimizer settings.
                    	    We use a complex Adam variant outlined in Algorithm~\ref{alg:complex_adam}, where we only tuned $\momCoeff_1$ for the discriminator.
                    	    The best parameters found in Figure~\ref{fig:biggan_inception_heatmap} were $\momCoeff_1 = \num{.8}\exp(i \nicefrac{\pi}{\num{8}})$, which improved the min and max IS from our runs of the BigGAN authors baseline, which was the SoTA optimizer in this setting to best of our knowledge.
                    	    We tested $\momCoeff_1 = \num{.8}$ to see if the gain was solely from tuning $|\momCoeff_1|$, which occasionally failed and decreased the best IS.
                    	 }
                    	\label{tab:biggan}
                    \end{table}
                \end{minipage}
                \vspace{-0.03\textheight}
            \end{figure}
            %
            We tested the best momentum value over $\num{10}$ seeds against the author-provided baseline in Appendix Figure~\ref{fig:biggan_inception_runs}, with the results summarized in Table~\ref{tab:biggan}.
            \citep{brock2018large} reported a single inception score (IS) on CIFAR-10 of $\num{9.22}$, but the best we could reproduce over the seeds with the provided PyTorch code and settings was $\num{9.10}$.
            Complex momentum improves the best IS found with $\num{9.25} ({\color{green}+.15} \text{ over author code}, {\color{green}+.03} \text{ author reported})$.
            We trained a real momentum $|\momCoeff_1| \!=\! \num{0.8}$ to see if the improvement was solely from tuning the momentum magnitude.
            This occasionally failed to train and decreased the best IS over re-runs, showing we benefit from a non-zero $\arg(\momCoeff_1)$.

            %
    \vspace{-0.0125\textheight}
    \subsection{A Practical Initial Guess for Optimizer Parameter $\arg(\momCoeff)$}\label{sec:init_guess}
    \vspace{-0.0125\textheight}
        \new{
        Here, we propose a practical initial guess for our new hyperparameter $\arg(\momCoeff)$.
        \resultName~\ref{thm:theorem_existence} shows we can use almost-real momentum coefficients (i.e., $\arg(\momCoeff)$ is close to $\num{0}$).
        Figure~\ref{fig:partial_coop_phases} shows almost-positive $\momCoeff$ approach acceleration in cooperative eigenspaces, while converging in all eigenspaces.
        Figure~\ref{fig:gan_spectrum_main} shows GANs can have both cooperative and adversarial eigenspaces.
        Figures~\ref{fig:gan_heatmaps} and \ref{fig:biggan_inception_heatmap} do a grid search over $\arg(\momCoeff)$ for GANs, finding that almost-positive $\arg(\momCoeff) \approx \nicefrac{\pi}{\num{8}}$ works in both cases.
        Also, by minimally changing $\arg(\momCoeff)$ from $0$ to a small $\epsilon$, we can minimally change other hyperparameters in our model, which is useful to adapt existing, brittle setups like in GANs.
        Based on this, we propose an initial guess of $\arg(\momCoeff) = \epsilon$ for a small $\epsilon > 0$, where $\epsilon = \nicefrac{\pi}{\num{8}}$ worked in our GAN experiments.
        }

    \newcommand{\ytickspace}{0.042\textwidth}
    \begin{wrapfigure}[36]{r}{0.585\linewidth}
        \vspace{-0.0525\textheight}
        \centering
        \begin{tikzpicture}
            \centering
            \node (img1){\includegraphics[trim={1.3cm .5cm .9cm .475cm},clip,width=.75\linewidth]{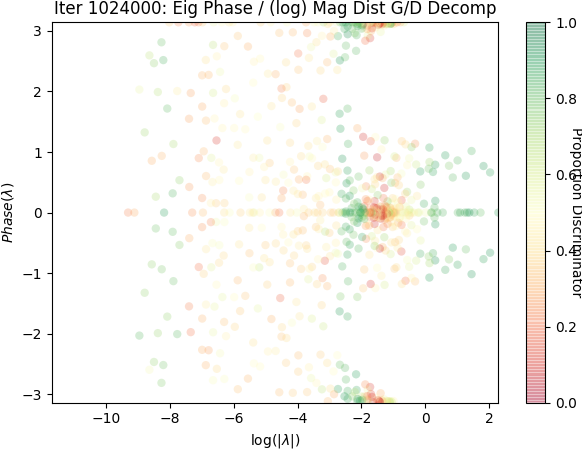}};
            \node (img1ytick)[left=of img1, node distance=0cm, rotate=90, xshift=2.67cm, yshift=-0.825cm, font=\color{black}] {-$\pi$ \hspace{\ytickspace} -$\frac{\pi}{2}$ \hspace{\ytickspace} $\num{0}$ \hspace{\ytickspace} \hphantom{-}$\frac{\pi}{2}$ \hspace{\ytickspace} \hphantom{-}$\pi$};
            \node[left=of img1ytick, node distance=0cm, rotate=90, xshift=4.6cm, yshift=-0.53cm, font=\color{black}] {Phase of eigenvalue $\arg(\eigval)$};
            
            \node[above=of img1, node distance=0cm, xshift=-.0cm, yshift=-1.2cm,font=\color{black}] {Spectrum of Jacobian of joint-grad $\spectrum(\nabla_{\bothParam}\bothGrad^{j})$ for GAN};
            \node[below=of img1, node distance=0cm, xshift=-.1cm, yshift=1.2cm,font=\color{black}] {Log-magnitude of eigenvalue $\log(|\eigval|)$};
            
            \node (img1colortick)[right=of img1, node distance=0cm, rotate=270, xshift=-2.6cm, yshift=-0.83cm, font=\color{black}] {disc. \hspace{0.07\textwidth} unsure \hspace{0.08\textwidth} gen.};
            \node[right=of img1colortick, node distance=0cm, rotate=270, xshift=-5cm, yshift=-.55cm, font=\color{black}] {Does eigenvector point at a player?};
        \end{tikzpicture}
        \vspace{-0.0275\textheight}
        \caption{
            A log-polar coordinate visualization reveals structure in the spectrum for a GAN at the end of the training on a $\num{2}$D mixture of Gaussians with a $\num{1}$-layer (disc)riminator and (gen)erator, so the joint-parameters $\bothParam \!\in\! \mathbb{R}^{\num{723}}$.
            It is difficult to see structure by graphing the Cartesian (i.e., $\Re$ and $\Im$) parts of eigenvalues, because they span orders of magnitude, while being positive and negative.
            Appendix Figure~\ref{fig:gan_decomp} shows the spectrum through training.
            \newline
            There is a mixture of many cooperative (i.e., real or $\arg(\eigval) \!\approx\! \num{0}, \pm \pi$) and some adversarial (i.e., imaginary or $\arg(\eigval) \!\approx\! \pm \frac{\pi}{\num{2}}$) eigenvalues, so -- contrary to what the name may suggest -- generative adversarial networks are not purely adversarial.
            We may benefit from optimizers leveraging this structure like complex momentum.
            \newline
            Eigenvalues are colored if the associated eigenvector is mostly in one player's part of the joint-parameter space -- see Appendix Figure~\ref{fig:gan_decomp} for details on this.
            Many eigenvectors lie mostly in the the space of (or point at) a one player.
            The structure of the set of eigenvalues for the disc. ({\color{green}green}) is different than the gen. ({\color{red}red}), but further investigation of this is an open problem.
            Notably, this may motivate separate optimizer choices for each player as in Section~\ref{sec:train_big_gan}.
        }
        \label{fig:gan_spectrum_main}
    \end{wrapfigure}
  \vspace{-0.0175\textheight}
  \section{Related Work}\label{sec:related-work}
  \vspace{-0.015\textheight}
        \textbf{Accelerated first-order methods}:
            A broad body of work exists using momentum-type methods~\citep{polyak1964some, nesterov1983method, nesterov2013introductory, maddison2018hamiltonian}, with a recent focus on deep learning~\citep{sutskever2013importance, zhang2017yellowfin, choi2019empirical, zhang2019lookahead, chen2020self}.
            But, these works focus on momentum for minimization as opposed to in games.
        
        \vspace{-0.005\textheight}
        \textbf{Learning in games}:
            Various works approximate response-gradients - some by differentiating through optimization~\citep{foerster2018learning, mescheder2017numerics, maclaurin2015gradient}.
            Multiple works try to leverage game eigenstructure during optimization~\citep{letcher2019differentiable, nagarajan2020chaos, omidshafiei2020navigating, czarnecki2020real, gidel2020minimax, perolat2020poincar}.
        
        \vspace{-0.005\textheight}   
        \textbf{First-order methods in games}:
            In some games, we can get away with using only first-order methods -- \citet{zhang2021don, zhang2020unified, ibrahim2020linear, bailey2020finite, jin2020local, azizian2019tight, nouiehed2019solving, zhang2020optimality} discuss when and how these methods work.
            \citet{gidel2018negative} is the closest work to ours, showing a negative momentum can help in some games.
            \citet{zhang2020suboptimality} note the suboptimality of negative momentum in a class of games.
            \citet{azizian2020accelerating, domingo2020average} investigate acceleration in some games.
        
        \vspace{-0.005\textheight} 
        \textbf{Bilinear zero-sum games}:
            \citet{zhangconvergence} study the convergence of gradient methods in bilinear zero-sum games.
            Their analysis extends \citet{gidel2018negative}, showing that we can achieve faster convergence by having separate step sizes and momentum for each player or tuning the extragradient step size.
            \citet{loizou2020stochastic} provide convergence guarantees for games satisfying a \emph{sufficiently bilinear} condition.
        
        \vspace{-0.005\textheight}
        \textbf{Learning in GANs}:
            Various works try to make GAN training easier with methods leveraging the game structure~\citep{Liu2020Towards, peng2020training, albuquerque2019multi, wu2019logan, hsieh2019finding}.
            \citet{metz2016unrolled} approximate the discriminator's response function by differentiating through optimization.
            \citet{mescheder2017numerics} find solutions by minimizing the norm of the players' updates.
            Both of these methods and various others~\citep{qin2020training, schafer2019implicit, jolicoeur2019connections} require higher-order information.
            \citet{daskalakis2017training, gidel2018variational, chavdarova2019reducing} look at first-order methods.
            \citet{mescheder2018training} explore problems for GAN training convergence and \citet{berard2019closer} show that GANs have significant rotations affecting learning.
        
    \vspace{-0.015\textheight}
    \section{Conclusion}\label{sec:conclusion}
    \vspace{-0.015\textheight}
        \new{
        In this paper we provided a generalization of existing momentum methods for learning in differentiable games by allowing a complex-valued momentum with real-valued updates.
        We showed that our method robustly converges in games with a different range of mixtures of cooperative and adversarial eigenspaces than current first-order methods.
        We also presented a practical generalization of our method to the Adam optimizer, which we used to improve BigGAN training.
        More generally, we highlight and lay groundwork for investigating optimizers which work well with various mixtures of cooperative and competitive dynamics in games.
        %
        }
    \new{
        \vspace{-0.0125\textheight}
        \subsection*{Societal Impact}
        \vspace{-0.0125\textheight}
            Our main contribution in this work is methodological -- specifically, a scalable algorithm for optimizing in games.
            Since our focus is on improving optimization methods, we do not expect there to be direct negative societal impacts from this contribution.
    }
    \newpage
    \clearpage
    \newpage
        
    \subsection*{Acknowledgements}
        Resources used in preparing this research were provided, in part, by the Province of Ontario, the Government of Canada through CIFAR, and companies sponsoring the Vector Institute.
        Paul Vicol was supported by an NSERC PGS-D Scholarship.
        We thank Guodong Zhang, Guojun Zhang, James Lucas, Romina Abachi, Jonah Phillion, Will Grathwohl, Jakob Foerster, Murat Erdogdu, Ken Jackson, and Ioannis Mitliagkis for feedback and helpful discussion.

    {\small
    \bibliography{references}
    }
    \clearpage
    \newpage
    \appendix
    \title{Appendix: \localTitle }

        \newpage
    
            \begin{table}[h!]
                \vspace{-0.02\textheight}
                \caption{Notation}
                \begin{center}
                    \begin{tabular}{c c}
                        \toprule
                        SGD & Stochastic Gradient Descent\\
                        CM & Complex Momentum\\
                        SGDm, SimSGDm, \dots & \dots with momentum\\
                        SimSGD, SimCM & Simultaneous \dots\\
                        AltSGD, AltCM & Alternating \dots\\
                        GAN & Generative Adversarial Network~\citep{goodfellow2014generative}\\
                        EG & Extragradient~\citep{korpelevich1976extragradient}\\
                        OG & Optimistic Gradient~\citep{daskalakis2017training}\\
                        IS & Inception Score~\citep{salimans2016improved}\\
                        $\defeq$ & Defined to be equal to\\
                        $x, y, z, \dots \in \mathbb{C}$ & Scalars\\
                        $\boldsymbol{x}, \boldsymbol{y}, \boldsymbol{z}, \dots \in \mathbb{C}^{n}$ & Vectors\\
                        $\boldsymbol{X}, \boldsymbol{Y}, \boldsymbol{Z}, \dots \in \mathbb{C}^{n \times n}$ & Matrices\\
                        $\boldsymbol{X}^\transpose$ & The transpose of matrix $\boldsymbol{X}$\\
                        $\identity$ & The identity matrix\\
                        $\Re(z), \Im(z)$ & The real or imaginary component of $z \in \mathbb{C}$\\
                        $i$ & The imaginary unit. $z \in \mathbb{C} \implies z = \Re(z) + i \Im(z)$\\
                        $\widebar{z}$ & The complex conjugate of $z \in \mathbb{C}$\\
                        $|z| \defeq \sqrt{z \widebar{z}}$ & The magnitude or modulus of $z \in \mathbb{C}$\\
                        $\arg(z)$ & The argument or phase of $z \in \mathbb{C} \implies z = |z| \exp(i\arg(z))$\\
                        $z \!\in\! \mathbb{C}$ is \emph{almost-positive}& $\arg(z) \!=\! \epsilon$ for small $\epsilon$ respectively\\
                        $\outSymbol, \inSymbol$ & A symbol for the outer/inner players\\
                        $\outDim, \inDim \in \mathbb{N}$ & The number of weights for the outer/inner players\\
                        $\paramSymbol$ & A symbol for the parameters or weights of a player\\
                        $\outParam \in \mathbb{R}^{\outDim}, \inParam \in \mathbb{R}^{\inDim}$ & The outer/inner parameters or weights\\
                        $\loss: \mathbb{R}^{n} \to \mathbb{R}$ & A symbol for a loss\\
                        $\outLoss(\outParam, \inParam), \inLoss(\outParam, \inParam)$ & The outer/inner losses -- $\mathbb{R}^{\outDim + \inDim} \mapsto \mathbb{R}$\\
                        $\outGrad(\outParam, \inParam), \inGrad(\outParam, \inParam)$ & Gradient of outer/inner losses w.r.t. their weights in $\mathbb{R}^{\outDim/\inDim}$\\
                        $\inParam^*\!(\outParam\!) \!\defeq\! \argmin\limits_{\inParam} \!\!\inLoss\!(\!\outParam \!, \!\inParam\!)$&The best-response of the inner player to the outer player\\
                        $\outLoss^*(\outParam) \!\defeq\! \outLoss\!(\!\outParam \!, \!\inParam^*(\outParam)\!)$ & The outer loss with a best-responding inner player\\
                        $\outParam^* \!\defeq\! \argmin\limits_{\outParam} \outLoss^*(\outParam\!) $ & Outer optimal weights with a best-responding inner player\\
                        $\bothDim \defeq \outDim + \inDim$ & The combined number of weights for both players\\
                        $\bothParam \defeq [\outParam, \inParam] \in \mathbb{R}^{\bothDim}$ & A concatenation of the outer/inner weights\\
                        $\bothGrad(\bothParam) \!\defeq\! [\outGrad(\bothParam), \inGrad(\bothParam)] \in \mathbb{R}^{\bothDim}$ & A concatenation of the outer/inner gradients\\
                        $\bothParam^{\num{0}} = [\outParam^{\num{0}}, \inParam^{\num{0}}] \in \mathbb{R}^{\bothDim}$ & The initial parameter values\\
                        $j$ & An iteration number\\
                        $\bothGrad^j \defeq \bothGrad(\bothParam^j) \in \mathbb{R}^{\bothDim}$ & The joint-gradient vector field at weights $\bothParam^j$\\
                        $\nabla_{\bothParam} \bothGrad^j \defeq \nabla_{\bothParam}\bothGrad |_{\bothParam^j} \in \mathbb{R}^{\bothDim \times \bothDim}$ & The Jacobian of the joint-gradient $\bothGrad$ at weights $\bothParam^j$\\
                        $\lr \in \mathbb{C}$ & The step size or learning rate\\
                        $\momCoeff \in \mathbb{C}$ & The momentum coefficient\\
                        $\momCoeff_1 \in \mathbb{C}$ & The first momentum parameter for Adam\\
                        $\momVal \in \mathbb{C}^{\bothDim}$ & The momentum buffer\\
                        $\eigval \in \mathbb{C}$ & Notation for an arbitrary eigenvalue\\
                        $\spectrum(\boldsymbol{M}) \in \mathbb{C}^{n}$ & The spectrum -- or set of eigenvalues -- of $\boldsymbol{M} \in \mathbb{R}^{n \times n}$\\
                        \emph{Purely adversarial/cooperative} game& $\spectrum(\nabla_{\bothParam} \bothGrad)$ is purely real/imaginary\\
                        $\spectralRadius(\boldsymbol{M}) \!\defeq\! \max_{z \in \spectrum(\boldsymbol{M})} |z|$ & The spectral radius in $\mathbb{R}^{+}$ of $\boldsymbol{M} \in \mathbb{R}^{n \times n}$\\
                        $\fixedPointOp([\momVal, \bothParam])$ & Fixed point op. for CM, or augmented learning dynamics\\
                        $\dynamics \defeq \nabla_{[\momVal, \bothParam]}\fixedPointOp \in \mathbb{R}^{\num{3}\bothDim \times \num{3}\bothDim}$ & Jacobian of the augmented learning dynamics in \resultName~\ref{thm:theorem}\\
                        $\lr^*\!, \momCoeff^* \!\defeq\! \argmin\limits_{\lr, \momCoeff} \spectralRadius(\dynamics(\lr,\! \momCoeff)\!)$ & The optimal step size and momentum coefficient\\
                        $\spectralRadius^* \defeq \spectralRadius(R(\lr^*, \momCoeff^*))$ & The optimal spectral radius or convergence rate\\
                        $\condition \defeq \frac{\max \spectrum(\nabla_{\bothParam} \gradSymbol)}{\min \spectrum(\nabla_{\bothParam} \gradSymbol)}$ & Condition number, for convex single-objective optimization\\
                        $\sigma_{min}^2(\boldsymbol{M}) \!\defeq\! \max \spectrum(\boldsymbol{M}^\transpose\boldsymbol{M}) $ & The minimum singular value of a matrix $\boldsymbol{M}$\\
                        \bottomrule
                    \end{tabular}
                \end{center}
                \label{tab:TableOfNotation}
            \end{table}
        
    \section{Supporting Results}
        First, some basic results about complex numbers that are used:
        \begin{equation}
            z = \Re(z) + i \Im(z) = |z| \exp(i \arg(z))
        \end{equation}
        \begin{equation}
             \widebar{z} =  \Re(z) - i \Im(z) = |z| \exp(- i \arg(z))
        \end{equation}
        \begin{equation}
            \exp(i z) + \exp(-i z) = 2 \cos(z)
        \end{equation}
        \begin{equation}
            \widebar{z_1 z_2} = \widebar{z}_1 \widebar{z}_2
        \end{equation}
        \begin{equation}
            \nicefrac{1}{2}(z + \widebar{z}) = \Re(z)
        \end{equation}
        \begin{equation}
            \Re(z_1 z_2) = \Re(z_1) \Re(z_2) - \Im(z_1) \Im(z_2)
        \end{equation}
        \begin{equation}\label{eq:complex_add}
            z_1 + z_2 = (\Re(z_1) + \Re(z_2)) + i(\Im(z_1) + \Im(z_2))
        \end{equation}
        \begin{equation}\label{eq:complex_mult}
            z_1 z_2 = (\Re(z_1) \Re(z_2) - \Im(z_1)\Im(z_2)) + i(\Im(z_1)\Re(z_2) + \Re(z_1)\Im(z_2))
        \end{equation}
        \begin{equation}\label{eq:complex_mult_polar}
            z_1 z_2 = |z_1| |z_2| \exp(i (\arg(z_1) + \arg(z_2)))
        \end{equation}
        \begin{equation}\label{eq:euler_formula}
            z^k = |z|^k \exp(i \arg(z)k) = |z|^k (\cos(k\arg(z)) + i \sin(k\arg(z))
        \end{equation}
        \new{
        This Lemma shows how we expand the complex-valued momentum buffer $\momVal$ into its Cartesian components as in (\ref{eq:cartesian_complex_update}).
        }
        \begin{lemma}\label{lemma:complex_momentum_cartesian_buffer}
            \begin{align*}
                \momVal^{j\!+\!1} &= \momCoeff \momVal^{j} - \bothGrad^j \iff\\
                \Re(\momVal^{j\!+\!1}) &= \Re(\momCoeff) \Re(\momVal^j) - \Im(\momCoeff) \Im(\momVal^j)  - \Re(\bothGrad^j), \Im(\momVal^{j\!+\!1}) = \Im(\momCoeff) \Re(\momVal^{j}) + \Re(\momCoeff) \Im(\momVal^{j}) - \Im(\bothGrad^j)
            \end{align*}
        \end{lemma}
        \begin{proof}
            \begin{align*}
                \momVal^{j\!+\!1} &= \momCoeff \momVal^{j} - \bothGrad^j\\
                \iff \momVal^{j\!+\!1} &= \left(\Re(\momCoeff) + i\Im(\momCoeff) \right) \left( \Re(\momVal^{j}) + i\Im(\momVal^{j}) \right) - \left( \Re(\bothGrad^j) + i\Im(\bothGrad^j) \right)\\
                \iff \momVal^{j\!+\!1} &= \left(\Re(\momCoeff) \Re(\momVal^j) - \Im(\momCoeff) \Im(\momVal^j) \right) +\\
                &\hspace{0.25\linewidth}i \left(\Im(\momCoeff) \Re(\momVal^{j}) + \Re(\momCoeff) \Im(\momVal^{j}) \right) - \left( \Re(\bothGrad^j) + i\Im(\bothGrad^j) \right)\\
                \iff \momVal^{j\!+\!1} &= \left(\Re(\momCoeff) \Re(\momVal^j) - \Im(\momCoeff) \Im(\momVal^j)  - \Re(\bothGrad^j) \right) +\\
                &\hspace{0.25\linewidth}i \left(\Im(\momCoeff) \Re(\momVal^{j}) + \Re(\momCoeff) \Im(\momVal^{j}) - \Im(\bothGrad^j)\right)\\
                \iff \Re(\momVal^{j\!+\!1}) \!&=\! \Re(\momCoeff) \Re(\momVal^j) \!-\! \Im(\momCoeff) \Im(\momVal^j)  \!-\! \Re(\bothGrad^j), \Im(\momVal^{j\!+\!1}) \!=\! \Im(\momCoeff) \Re(\momVal^{j}) \!+\! \Re(\momCoeff) \Im(\momVal^{j}) \!-\! \Im(\bothGrad^j)\\
            \end{align*}
        \end{proof}

        We further assume $\Im(\bothGrad^j)$ is $\num{0}$ - i.e., our gradients are real-valued.
        \new{
        This Lemma shows how we can decompose the joint-parameters $\bothParam$ at the next iterate as a linear combination of the joint-parameters, joint-gradient, and Cartesian components of the momentum-buffer at the current iterate as in (\ref{eq:parameter_dynamics}).
        }
        \begin{lemma}\label{lemma:complex_momentum_cartesian_parameters}
            \begin{align*}
                \bothParam^{j\!+\!1} = \bothParam^j + \Re(\lr \momVal^{j\!+\!1})
                \iff \bothParam^{j\!+\!1} = \bothParam^j - \Re(\lr)\bothGrad^j + \Re(\lr \momCoeff) \Re(\momVal^j) - \Im(\lr \momCoeff)\Im(\momVal^j)
            \end{align*}
        \end{lemma}
        \begin{proof}
            \begin{align*}
                & \Re(\lr \momVal^{j\!+\!1})\\
                =& \left(\Re(\lr) \Re(\momVal^{j\!+\!1}) - \Im(\lr) \Im(\momVal^{j\!+\!1})\right)\\
                =& \left(\Re(\lr) \left(\Re(\momCoeff) \Re(\momVal^j) - \Im(\momCoeff) \Im(\momVal^j) - \bothGrad^j\right) - \Im(\lr) \left( \Im(\momCoeff) \Re(\momVal^{j}) + \Re(\momCoeff) \Im(\momVal^{j}) \right)\right)\\
                =& -\Re(\lr)\bothGrad^j +  \left(\Re(\lr) \left(\Re(\momCoeff) \Re(\momVal^j) - \Im(\momCoeff) \Im(\momVal^j) \right) - \Im(\lr) \left( \Im(\momCoeff) \Re(\momVal^{j}) + \Re(\momCoeff) \Im(\momVal^{j}) \right)\right)\\
                =& -\Re(\lr)\bothGrad^j +  \left(\Re(\lr)\Re(\momCoeff) - \Im(\lr)\Im(\momCoeff) \right)\Re(\momVal^j) - \left(\Re(\lr) \Im(\momCoeff) + \Im(\lr)\Re(\momCoeff)\right)\Im(\momVal^j)\\
                =& -\Re(\lr)\bothGrad^j +  \Re(\lr \momCoeff) \Re(\momVal^j) - \Im(\lr \momCoeff)\Im(\momVal^j)
            \end{align*}
            Thus,
            \begin{align*}
                \bothParam^{j\!+\!1} = \bothParam^j + \Re(\lr \momVal^{j\!+\!1})
                \iff \bothParam^{j\!+\!1} = \bothParam^j - \Re(\lr)\bothGrad^j + \Re(\lr \momCoeff) \Re(\momVal^j) - \Im(\lr \momCoeff)\Im(\momVal^j)
            \end{align*}
        \end{proof}
        
        \subsection{Theorem~\ref{thm:theorem} Proof Sketch}\label{pf:thm1}
            \mainThm*
            \begin{proof}
                We reproduce the proof for a simpler case of quadratic games, which is simple case of \citet{polyak1964some}'s well-known method for analyzing the convergence of iterative methods.
                \citet{bertsekas2008nonlinear} generalizes this result from quadratic games to when we are sufficiently close to any stationary point.
                
                For quadratic games, we have that $\bothGrad^j = \left(\nabla_{\bothParam} \bothGrad\right)^\transpose \bothParam^j$.
            	Well, by Lemma~\ref{lemma:complex_momentum_cartesian_buffer} and Lemma~\ref{lemma:complex_momentum_cartesian_parameters} we have:
            	\begin{equation}
            	    \begin{pmatrix}
            	        \Re(\momVal^{j\!+\!1})\\
            	        \Im(\momVal^{j\!+\!1})\\
            	        \bothParam^{j\!+\!1}
            	    \end{pmatrix}
            	    = \dynamics 
            	    \begin{pmatrix}
            	        \Re(\momVal^{j})\\
            	        \Im(\momVal^{j})\\
            	        \bothParam^{j}
            	    \end{pmatrix}
            	\end{equation}
            	By telescoping the recurrence for the $j^{th}$ augmented parameters: 
            	\begin{equation}
            	    \begin{pmatrix}
            	        \Re(\momVal^{j})\\
            	        \Im(\momVal^{j})\\
            	        \bothParam^{j}
            	    \end{pmatrix}
            	    = \dynamics^j
            	    \begin{pmatrix}
            	        \Re(\momVal^{0})\\
            	        \Im(\momVal^{0})\\
            	        \bothParam^{0}
            	    \end{pmatrix}
            	\end{equation}
            	We can compare $\momVal^j$ with the value it converges to $\momVal^*$ which exists if $\dynamics$ is contractive.
            	We do the same with $\bothParam$.
            	Because $\momVal^* = \dynamics \momVal^* = \dynamics^j \momVal^*$:
            	\begin{equation}
            	    \begin{pmatrix}
            	        \Re(\momVal^{j}) - \Re(\momVal^{*})\\
            	        \Im(\momVal^{j}) - \Im(\momVal^{*})\\
            	        \bothParam^{j} - \bothParam^{*}
            	    \end{pmatrix}
            	    = \dynamics^j
            	    \begin{pmatrix}
            	        \Re(\momVal^{0})  - \Re(\momVal^{*})\\
            	        \Im(\momVal^{0})  - \Im(\momVal^{*})\\
            	        \bothParam^{0} - \bothParam^{*}
            	    \end{pmatrix}
            	\end{equation}
            	By taking norms:
            	\begin{align}
            	    \left\|
                	    \begin{pmatrix}
                	        \Re(\momVal^{j}) - \Re(\momVal^{*})\\
                	        \Im(\momVal^{j}) - \Im(\momVal^{*})\\
                	        \bothParam^{j} - \bothParam^{*}
                	    \end{pmatrix}
            	    \right\|_2
            	    &= 
            	    \left\|
                	    \dynamics^j
                	    \begin{pmatrix}
                	        \Re(\momVal^{0})  - \Re(\momVal^{*})\\
                	        \Im(\momVal^{0})  - \Im(\momVal^{*})\\
                	        \bothParam^{0} - \bothParam^{*}
                	    \end{pmatrix}
            	    \right\|_2
            	    \\
            	    \implies
            	    \left\|
                	    \begin{pmatrix}
                	        \Re(\momVal^{j}) - \Re(\momVal^{*})\\
                	        \Im(\momVal^{j}) - \Im(\momVal^{*})\\
                	        \bothParam^{j} - \bothParam^{*}
                	    \end{pmatrix}
            	    \right\|_2
            	    &\leq
            	    \left\|
                	    \dynamics^j
                	\right\|_2
                	\left\|
                	    \begin{pmatrix}
                	        \Re(\momVal^{0})  - \Re(\momVal^{*})\\
                	        \Im(\momVal^{0})  - \Im(\momVal^{*})\\
                	        \bothParam^{0} - \bothParam^{*}
                	    \end{pmatrix}
            	    \right\|_2
            	\end{align}
            	With Lemma 11 from \citet{foucart2012}, we have there exists a matrix norm $\forall \epsilon > 0$ such that:
            	\begin{equation}\label{eq:foucart}
            	    \| \dynamics^j \| \leq \left(\spectralRadius \left(\dynamics \right) + \epsilon\right)^j
            	\end{equation}
            	We also have an equivalence of norms in finite-dimensional spaces.
            	So for all norms $\| \cdot \|$, $\exists C \geq B > 0$ such that:
            	\begin{equation}\label{eq:norm_eq}
            	    B \| \dynamics^j \| \leq \| \dynamics^j \|_2 \leq C \| \dynamics^j \|
            	\end{equation}
            	Combining (\ref{eq:foucart}) and (\ref{eq:norm_eq}) we have:
            	\begin{equation}
            	    \left\|
                	    \begin{pmatrix}
                	        \Re(\momVal^{j}) - \Re(\momVal^{*})\\
                	        \Im(\momVal^{j}) - \Im(\momVal^{*})\\
                	        \bothParam^{j} - \bothParam^{*}
                	    \end{pmatrix}
            	    \right\|_2
            	    \leq
            	    C \left(\spectralRadius \left(\dynamics \right) + \epsilon\right)^j
                	\left\|
                	    \begin{pmatrix}
                	        \Re(\momVal^{0})  - \Re(\momVal^{*})\\
                	        \Im(\momVal^{0})  - \Im(\momVal^{*})\\
                	        \bothParam^{0} - \bothParam^{*}
                	    \end{pmatrix}
            	    \right\|_2
            	\end{equation}
            	So, we have:
            	\begin{equation}
            	    \left\|
                	    \begin{pmatrix}
                	        \Re(\momVal^{j}) - \Re(\momVal^{*})\\
                	        \Im(\momVal^{j}) - \Im(\momVal^{*})\\
                	        \bothParam^{j} - \bothParam^{*}
                	    \end{pmatrix}
            	    \right\|_2
            	    = \mathcal{O}((\spectralRadius(\dynamics) + \epsilon)^j)
            	\end{equation}
            	Thus, we converge linearly with a rate of $\mathcal{O}(\spectralRadius(\dynamics) + \epsilon)$.
            \end{proof}
        \subsection{Characterizing the Augmented Dynamics Eigenvalues}\label{sec:poly}
            Here, we present polynomials whose roots are the eigenvalues of our the Jacobian of our augmented dynamics $\spectrum(\dynamics)$, given the eigenvalues of the Jacobian of the joint-gradient vector field $\spectrum(\nabla_{\bothParam} \bothGrad)$.
            We use a similar decomposition as \citet{gidel2018negative}.
            
            We can expand $\nabla_{\bothParam} \bothGrad = P T P^{-1}$ where $T$ is an upper-triangular matrix and $\eigval_i$ is an eigenvalue of $\nabla_{\bothParam} \bothGrad$.
            \begin{equation}
                T = \begin{bmatrix}
                    \eigval_1 & * & \dots & *\\
                    \num{0}  & \dots & \dots & \dots\\
                    \dots & \dots & \dots & *\\
                    \num{0} & \dots & \num{0} & \eigval_{\bothDim}\\
                    \end{bmatrix}
            \end{equation}
            We then break up into components for each eigenvalue, giving us submatrices $\mathbf{R}_k \in \mathbb{C}^{\num{3} \times \num{3}}$:
            \begin{equation}\label{eq:R-decomp}
        	    \dynamics_{k}  \defeq \begin{bmatrix}
                    \Re(\momCoeff) & -\Im(\momCoeff) & -\eigval_k\\
                    \Im(\momCoeff) & \Re(\momCoeff) & 0\\
                    \Re(\lr \momCoeff) & -\Im(\lr \momCoeff) & 1 - \Re(\lr)\eigval_k\\
                    \end{bmatrix}
        	\end{equation}
        	We can get the characteristic polynomial of $\dynamics_k$ with the following Mathematica command, where we use substitute the symbols $r + i u = \eigval_k$, $a = \Re(\momCoeff)$, $b = \Im(\momCoeff)$, $c = \Re(\lr)$, and $d = \Im(\lr)$.
        	
        	\texttt{CharacteristicPolynomial[\{\{a, -b, -(r + u I)\}, \{b, a, 0\}, \{a c - b d, -(b c + a d), 1 - c (r + u I)\}\}, x]}
        	
        	The command gives us the polynomial associated with eigenvalue $\eigval_k = r + i u$:
        	\begin{equation}\label{eq:R-poly}
        	    p_k(x) = -a^2 x + a^2 + a c r x + i a c u x + 2 a x^2 - 2 a x - b^2 x + b^2 + b d r x + i b d u x - c r x^2 - i c u x^2 - x^3 + x^2
        	\end{equation}
        	Consider the case where $\eigval_k$ is imaginary -- i.e, $r = 0$ -- which is true in all purely adversarial and bilinear zero-sum games.
        	Then (\ref{eq:R-poly}) simplifies to:
        	\begin{equation}\label{eq:R-poly-imEig}
        	    p_k(x) = -a^2 x + a^2 + i a c u x + 2 a x^2 - 2 a x - b^2 x + b^2 + i b d u x - i c u x^2 - x^3 + x^2
        	\end{equation}
        	Our complex $\eigval_k$ come in conjugate pairs where $\eigval_k = u_k i$ and $\bar{\eigval}_k = - u_k i$.
        	(\ref{eq:R-poly-imEig}) has the same roots for $\eigval_k$ and $\bar{\eigval}_k$, which can be verified by writing the roots with the cubic formula.
        	This corresponds to spiraling around the solution in either a clockwise or counterclockwise direction.
        	Thus, we restrict to analyzing $\eigval_k$ where $u_k$ is positive without loss of generality.
        	
        	If we make the step size $\lr$ real -- i.e., $d = 0$ -- then (\ref{eq:R-poly-imEig}) simplifies to:
        	\begin{equation}\label{eq:R-poly-imEig-realLr}
        	    p_k(x) = x (-a^2 + i a c u - 2 a - b^2) + a^2 + x^2 (2 a - i c u + 1) + b^2 - x^3
        	\end{equation}
        	Using a heuristic from single-objective optimization, we look at making step size proportional to the inverse of the magnitude of eigenvalue $k$ -- i.e., $\lr_k = \frac{\lr'}{|\eigval_k|} = \frac{\lr'}{u_k}$.
        	With this, (\ref{eq:R-poly-imEig-realLr}) simplifies to:
        	\begin{equation}\label{eq:R-poly-imEig-realLr-prop}
        	    p_k(x) = x (-a^2 + i a \lr' - 2 a - b^2) + a^2 + x^2 (2 a - i \lr' + 1) + b^2 - x^3
        	\end{equation}
        	Notably, in (\ref{eq:R-poly-imEig-realLr-prop}) there is no dependence on the components of imaginary eigenvalue $\eigval_k = r + i u = \num{0} + i u$, by selecting a $\lr$ that is proportional to the eigenvalues inverse magnitude.
        	We can simplify further with $a^2 + b^2 = |\momCoeff|^2$:
        	\begin{equation}\label{eq:R-poly-imEig-realLr-prop-s1}
        	    p_k(x) = x (\Re(\momCoeff) (i \lr' - 2) - |\momCoeff|^2) + x^2 (2 \Re(\momCoeff) - i \lr' + 1) + |\momCoeff|^2 - x^3
        	\end{equation}
        	We could expand this in polar form for $\momCoeff$ by noting $\Re(\momCoeff) = |\momCoeff| \cos(\arg(\momCoeff))$:
        	\begin{equation}\label{eq:R-poly-imEig-realLr-prop-s2}
        	    p_k(x) = x (|\momCoeff| \cos(\arg(\momCoeff))(i  \lr' - 2) - |\momCoeff|^2) + x^2 (2 |\momCoeff| \cos(\arg(\momCoeff)) - i \lr' + 1) + |\momCoeff|^2 - x^3
        	\end{equation}
        	We can simplify further by considering an imaginary $\momCoeff$ -- i.e., $\Re(\momCoeff) = 0$ or $\cos(\arg(\momCoeff)) = 0$:
        	\begin{equation}\label{eq:R-poly-imEig-realLr-prop-s3}
        	    p_k(x) = |\momCoeff|^2 - x |\momCoeff|^2 - x^2 (i \lr' - 1)  - x^3
        	\end{equation}

        	
        	
        	The roots of these polynomials can be trivially evaluated numerically or symbolically with the by plugging in $\momCoeff, \lr,$ and $\eigval_k$ then using the cubic formula.
        	This section can be easily modified for the eigenvalues of the augmented dynamics for variants of complex momentum by defining the appropriate $\dynamics$ and modifying the Mathematica command to get the characteristic polynomial for each component, which can be evaluated if it is a sufficiently low degree using known formulas.
        	
        \subsection{Convergence Bounds}\label{sec:thm2}

            \newcommand{\pureImaginaryMomentumMag}{\num{.7}}
            \newcommand{\pureImaginaryMomentumLR}{\num{.7}}
            \newcommand{\realRate}{\num{.949}}
            \newcommand{\almostNegMomentumMag}{\num{.986}}
            \newcommand{\almostNegMomentumLR}{\num{.75}}
            \newcommand{\almostNegRate}{\num{0.9998}}
            \newcommand{\almostPosMomentumMag}{\num{.9}}
            \newcommand{\almostPosMomentumLR}{\num{.025}}
            \newcommand{\almostPosRate}{\num{0.973}}
                \resultExist*
                \begin{proof}
                    \new{
                    Note that Theorem~\ref{thm:theorem} bounds the convergence rate of Algorithm~\ref{alg:simultaneous_complex} by $\spectrum(\dynamics)$.
                    Also, (\ref{eq:R-poly-imEig-realLr}) gives a formula for 3 eigenvalues in $\spectrum(\dynamics)$ given $\lr, \momCoeff,$ and an eigenvalue $\eigval \in \spectrum(\nabla_{\bothParam}\bothGrad)$.
                    The formula works by giving outputting a cubic polynomial whose roots are eigenvalues of $\spectrum(\dynamics)$, which can be trivially evaluated with the cubic formula.
                    
                    We denote the $k^{th}$ eigenspace of $\spectrum(\nabla_{\bothParam}\bothGrad)$ with eigenvalue $\eigval_k = i c_k$ and $|c_1| \leq \dots \leq |c_n|$, because bilinear zero-sum games have purely imaginary eigenvalues due to $\nabla_{\bothParam}\bothGrad$ being antisymmetric.
                    Eigenvalues come in a conjugate pairs, where $\bar{\eigval_k} = i (-c_k)$
                    
                    If we select momentum coefficient $\momCoeff = |\momCoeff| \exp(i \arg(\momCoeff))$ and step size $\lr_k = \frac{\lr_k'}{|c_k|}$, and use that $ \eigval \in \spectrum(\nabla_{\bothParam}\bothGrad)$ are imaginary, then -- as shown in Appendix Section~\ref{sec:poly} -- (\ref{eq:R-poly-imEig-realLr}) simplifies to:
                	\begin{equation}\label{eq:R-poly-imEig-realLr-prop-pf}
                	    p_k(x) = x (|\momCoeff| \cos(\arg(\momCoeff))(i  \lr_k' - 2) - |\momCoeff|^2) + x^2 (2 |\momCoeff| \cos(\arg(\momCoeff)) - i \lr_k' + 1) + |\momCoeff|^2 - x^3
                	\end{equation}
                	So, with these parameter selections, the convergence rate of Algorithm~\ref{alg:simultaneous_complex} in the $k^{th}$ eigenspace is bounded by the largest root of (\ref{eq:R-poly-imEig-realLr-prop-pf}).
                	
                	First, consider $\arg(\momCoeff) = \pi - \epsilon$, where $\epsilon = \frac{\pi}{16}$.
                	We select $\lr_k' = \almostNegMomentumLR$ (equivalently, $\lr_k = \frac{\almostNegMomentumLR}{|c_k|}$) and $| \momCoeff | = \almostNegMomentumMag$ via grid search.
                    Using the cubic formula on the associated $p(x)$ from (\ref{eq:R-poly-imEig-realLr-prop-pf}) the maximum magnitude root has size $ \approx \almostNegRate < 1$, so this selection converges in the $k^{th}$ eigenspace.
                    So, selecting:
                    \begin{align}
                        \hat{\lr} &\leq \min_k \lr_k\\
                        &= \min_k \frac{\almostNegMomentumLR}{c_k}\\
                        &= \frac{\almostNegMomentumLR}{\max_k c_k}\\
                        &= \frac{\almostNegMomentumLR}{\|\nabla_{\bothParam}\bothGrad\|_2}
                    \end{align}
                     with $\momCoeff = \almostNegMomentumMag \exp(i (\pi - \epsilon))$ will converge in each eigenspace.
                    
                    Now, consider $\arg(\momCoeff) = \epsilon = \frac{\pi}{16}$ with $\lr_k' = \almostPosMomentumLR$ and $| \momCoeff | = \almostPosMomentumMag$.
                    Using the cubic formula on the associated $p(x)$ from (\ref{eq:R-poly-imEig-realLr-prop-pf}) the maximum magnitude root has size $ \approx \almostPosRate < 1$, so this selection converges in the $k^{th}$ eigenspace.
                    So, selecting:
                    \begin{align}
                        \hat{\lr} &\leq \min_k \lr_k\\
                        &= \min_k \frac{\almostPosMomentumLR}{c_k}\\
                        &= \frac{\almostPosMomentumLR}{\max_k c_k}\\
                        &= \frac{\almostPosMomentumLR}{\|\nabla_{\bothParam}\bothGrad\|_2}
                    \end{align}
                    with $\momCoeff = \almostPosMomentumMag \exp(i \epsilon)$ will converge in each eigenspace.
                    
                    Thus, for any of the choices of $\arg(\momCoeff)$ we can select $\hat{\lr}, |\momCoeff|$ that converges in every eigenspace, and thus converges.

                    }
                \end{proof}
                In the preceding proof, our prescribed selection of $\hat{\lr}$ depends on knowing the largest norm eigenvalue of $\spectrum(\nabla_{\bothParam}\bothGrad)$, because our selections of $\hat{\lr} \,\,\propto\,\, \frac{1}{\|\nabla_{\bothParam}\bothGrad\|_2}$.
                We may not have access to largest norm eigenvalue of $\spectrum(\nabla_{\bothParam}\bothGrad)$ in-practice.
                Nonetheless, this shows that a parameter selection exists to converge, even if it may be difficult to find.
                Often, in convex optimization we describe choices of $\lr, \momCoeff$ in terms of the largest and smallest norm eigenvalues of $\spectrum(\nabla_{\bothParam}\bothGrad)$ (i.e. the Hessian of the loss)~\citep{boyd2004convex}.

    \section{Algorithms}
        Here, we include additional algorithms, which may be of use to some readers.
        Algorithm~\ref{alg:aggmo} show aggregated momentum~\citep{lucas2018aggregated}.
        Algorithm~\ref{alg:recurrent_momentum} shows the recurrently linked momentum that generalizes and unifies aggregated momentum with negative momentum~\citep{gidel2018negative}.
        Algorithm~\ref{alg:alternating_complex} shows our algorithm with alternating updates, which we use for training GANs.
        Algorithm~\ref{alg:simultaneous_expanded} shows our method with all real-valued objects, if one wants to implement complex momentum in a library that does not support complex arithmetic.
        \begin{figure}[h!]
            \vspace{-0.02\textheight}
            \centering
            \begin{minipage}{0.4\textwidth}
                \begin{algorithm}[H]
                    \caption{Aggregated Momentum} \label{alg:aggmo}
                    \begin{algorithmic}[1]
                        \State Select number of buffers $K \in \mathbb{N}$
                        \State Select $\momCoeff_{(k)} \in [0, 1)$ for $k = 1 \dots K$
                        \State Select $\lr_{(k)} \in \mathbb{R}^{+}$ for $k = 1 \dots K$
                        \State Initialize $\momVal_{(k)}^{0}$ for $k = 1 \dots K$
                        \For{$j = \num{1} \dots N$}
                            \For{$k = \num{1} \dots K$}
                                \State $\momVal_{(k)}^{j\!+\!\num{1}} = \momCoeff_{(k)} \momVal_{(k)}^{j} - \bothGrad^{j}$
                            \EndFor
                            \State $\bothParam^{j\!+\!\num{1}} = \bothParam^j + \sum_{k=1}^{K}  \lr_{(k)} \momVal ^{j\!+\!1}_{(k)}$
                        \EndFor
                        \Return $\bothParam_N$
                    \end{algorithmic}
                \end{algorithm}
            \end{minipage}
            \hspace{0.01\textwidth}
            \begin{minipage}{0.55\textwidth}
                \begin{algorithm}[H]
                    \caption{Recurrently Linked Momentum} \label{alg:recurrent_momentum}
                    \begin{algorithmic}[1]
                        \State Select number of buffers $K \in \mathbb{N}$
                        \State Select $\momCoeff_{(l,k)} \in \mathbb{R}$ for $l=1 \dots K$ and $k = 1 \dots K$
                        \State Select $\lr_{(k)} \in \mathbb{R}^{+}$ for $k = 1 \dots K$
                        \State Initialize $\momVal_{(k)}^{0}$ for $k = 1 \dots K$
                        \For{$j = \num{1} \dots N$}
                            \For{$k = \num{1} \dots K$}
                                \State $\momVal_{(k)}^{j+1} \!=\! \sum_l \momCoeff_{(l, k)}\momVal_{(l)}^{j} - \bothGrad^{j}$
                            \EndFor
                            \State $\bothParam^{j\!+\!\num{1}} = \bothParam^j + \sum_{k=1}^{K} \lr_{(k)} \momVal ^{j\!+\!1}_{(k)}$
                        \EndFor
                        \Return $\bothParam_N$
                    \end{algorithmic}
                \end{algorithm}
            \end{minipage}
        \end{figure}
        
        \begin{figure}[h!]
            \vspace{-0.04\textheight}
            \centering
            \begin{minipage}{0.4\textwidth}
                \begin{algorithm}[H]
                    \caption{(\ref{eq:alt_update}) Momentum} \label{alg:alternating_complex}
                    \begin{algorithmic}[1]
                        \State Select $\momCoeff \in \mathbb{C}, \lr \in \mathbb{R}^{+}$
                        \State Initialize $\momVal_{\outSymbol}^{0}, \momVal_{\inSymbol}^{0}$
                        \For{$j = \num{1} \dots N$}
                            \State $\momVal_{\outSymbol}^{j\!+\!\num{1}} = \momCoeff \momVal_{\outSymbol}^{j} - \outGrad^j$
                            \State $\outParam^{j\!+\!\num{1}} = \outParam^j + \Re(\lr \momVal^{j\!+\!\num{1}}_{\outSymbol})$
                            \State $\momVal_{\inSymbol}^{j\!+\!\num{1}} = \momCoeff \momVal_{\inSymbol}^{j} - \inGrad(\outParam^{j\!+\!1}\!\!\!, \inParam^{j})$
                            \State $\inParam^{j\!+\!\num{1}} = \inParam^j + \Re(\lr \momVal^{j\!+\!1}_{\inSymbol})$
                        \EndFor
                        \Return $\bothParam_N$
                    \end{algorithmic}
                \end{algorithm}
            \end{minipage}
            %
            %
            %
            \hspace{0.01\textwidth}
            \begin{minipage}{0.55\textwidth}
                \centering
                \begin{algorithm}[H]
                    \caption{(\ref{eq:simul_update}) Complex Momentum - $\mathbb{R}$ valued} \label{alg:simultaneous_expanded}
                    \begin{algorithmic}[1]
                        \State Select $\Re(\momCoeff), \Im(\momCoeff), \Re(\lr), \Im(\lr) \in \mathbb{R}$
                        \State Select $\Re(\momCoeff), \Im(\momCoeff), \Re(\lr), \Im(\lr) \in \mathbb{R}$
                        \State Initialize $\Re(\momVal)^{0}, \Im(\momVal)^{0}$
                        \For{$j = 1 \dots N$}
                            \State $\Re(\momVal^{j\!+\!1}) = \Re(\momCoeff) \Re(\momVal^j) - \Im(\momCoeff)\Im(\momVal^j) - \bothGrad^j$
                            \State $\Im(\momVal^{j\!+\!1}) = \Re(\momCoeff) \Im(\momVal^j) + \Im(\momCoeff)\Re(\momVal^j)$
                            \State $\bothParam^{j\!+\!1} \!\!=\! \bothParam^{j} \!-\! \Re(\lr\!)\! \bothGrad^j \!+\! \Re(\lr \momCoeff\!)\!\Re(\momVal^{j}\!) \!-\! \Im(\lr \momCoeff\!)\!\Im(\momVal^{j}\!)$
                        \EndFor
                        \Return $\bothParam_N$
                    \end{algorithmic}
                \end{algorithm}
            \end{minipage}
            \vspace{-0.02\textheight}
        \end{figure}
    \subsection{Complex Momentum in PyTorch}
        Our method can be easily implemented in PyTorch 1.6+ by using complex tensors.
        The only necessary change to the SGD with momentum optimizer is extracting the real-component from momentum buffer as with JAX -- see \href{https://pytorch.org/docs/stable/_modules/torch/optim/sgd.html#SGD}{{\color{blue}here}}.
        
        In older versions of Pytorch, we can use a tensor to represent the momentum buffer $\momVal$, step size $\lr$, and momentum coefficient $\momCoeff$.
        Specifically, we represent the real and imaginary components of the complex number independently.
        Then, we redefine the operations \texttt{\_\_add\_\_} and \texttt{\_\_mult\_\_} to satisfy the rules of complex arithmetic -- i.e., equations (\ref{eq:complex_add}) and (\ref{eq:complex_mult}).
    \vspace{-0.01\textheight}
    \section{Experiments}\label{app:experiments}
    \vspace{-0.01\textheight}
        \newcommand{\bilinearTime}{}
        \newcommand{\supersmallGANTime}{}
        \newcommand{\spectrumTime}{}
        \newcommand{\smallGANTime}{2}
        \newcommand{\bigGANTime}{10}
        \subsection{Computing Infrastructure and Runtime}
            For the purely adversarial experiments in Sections \ref{sec:single_param_det} and \ref{sec:exp_spectrum}, we do our computing in CPU.
            Training each $\num{2}$D GAN in Section~\ref{sec:train_small_gan} takes $\smallGANTime$ hours and we can train $\num{10}$ simultaneously on an NVIDIA T$\num{4}$ GPU.
            Training each CIFAR GAN in Section~\ref{sec:train_big_gan} takes $\bigGANTime$ hours and we can only train $\num{1}$ model per NVIDIA T4 GPU.
        \subsection{Optimization in Purely Adversarial Games}\label{sec:app_bilinear}
            We include the alternating update version of Figure~\ref{fig:det_minmax_expAvg_simul_phase_tune} in Figure~\ref{fig:det_minmax_expAvg_alt_phase_tune}, which allows us to contrast simultaneous and alternating updates.
            With alternating updates on a Dirac-GAN for $\lr \!=\! \num{0.1}$ the best value for the momentum coefficient $\momCoeff$ was complex, but we could converge with real, negative momentum.
            Simultaneous updates may be a competitive choice with alternating updates, only if alternating updates cost two gradient evaluations per step, which is common in deep learning setups.

            \newcommand{\heatmapWidth}{.49\linewidth}
            \newcommand{\shiftHeatmapx}{-1.2cm}
            \newcommand{\shiftHeatmapy}{1.2cm}
        \subsection{How Adversarialness Affects Convergence Rates}\label{sec:app_spectrum}
            \vspace{-0.01\textheight}
            We include the extragradient (EG) update with extrapolation parameter $\lr'$ and step size $\lr$:
            \begin{align*}\label{eq:eg_update}\tag{EG}
                \bothParam^{j\!+\!\frac{1}{2}} \!=\! \bothParam^j \!-\! \lr' \bothGrad^{j}\\
                \bothParam^{j\!+\!1} \!=\! \bothParam^j \!-\! \lr \bothGrad^{j\!+\!\frac{1}{2}}\!
            \end{align*}
            and the optimistic gradient (OG) update with extrapolation parameter $\lr'$ and step size $\lr$:
            \begin{align*}\label{eq:eg_update}\tag{OG}
                \bothParam^{j\!+\!1} \!=\! \bothParam^j \!-\! 2 \lr \bothGrad^{j} \!+\! \lr' \bothGrad^{j-1}
            \end{align*}
            
            Often, EG and OG are used with $\lr = \lr'$, however we found that this constraint crippled these methods in cooperative games (i.e., minimization).
            As such, we tuned the extrapolation parameter $\lr'$ separately from the step size $\lr$, so EG and OG were competitive baselines.
            
            We include Figure~\ref{fig:gan_decomp} which investigates a GANs spectrum throughout training, and elaborates on the information that is shown in Figure~\ref{fig:gan_spectrum_main}.
            This shows that there are many real and imaginary eigenvalues, so GAN training is neither purely cooperative or purely adversarial.
            Also, the structure of the set of eigenvalues for the discriminator is different than the generator, which may motivate separate optimizer choices.
            The structure between the players persists through training, but the eigenvalues grow in magnitude and spread out their phases.
            This indicates how adversarial the game is can change during training.
            
            \begin{figure}
                \vspace{-0.02\textheight}
                \centering
                \begin{tikzpicture}
                    \centering
                    \node (img1){\includegraphics[trim={1.0cm .9cm 2.7cm .9cm},clip,width=.28\linewidth]{images/bilinear/ExpAvg_Simul_Deterministic_beta_tune.png}};
                    \node[left=of img1, node distance=0cm, rotate=90, xshift=1.75cm, yshift=-.9cm, font=\color{black}] {Momentum phase $\arg(\momCoeff)$};
                    \node[above=of img1, node distance=0cm, xshift=-.1cm, yshift=-1.3cm,font=\color{black}] {\small \ref{eq:simul_update} for \# grad. eval. = \# steps};
                    
                    \node [right=of img1, xshift=-1.1cm](img2){\includegraphics[trim={1.85cm .9cm .97cm .9cm},clip,width=.3\linewidth]{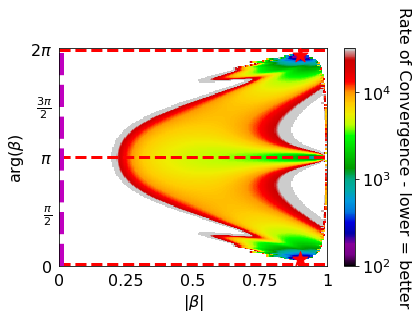}};
                    \node[below=of img2, node distance=0cm, xshift=.0cm, yshift=1.2cm,font=\color{black}] {Momentum Magnitude $| \momCoeff |$};
                    \node[above=of img2, node distance=0cm, xshift=-.3cm, yshift=-1.3cm,font=\color{black}] {\small \ref{eq:alt_update} for \# grad. eval.};
                    \node[right=of img2, node distance=0cm, rotate=270, xshift=-1.9cm, yshift=-.9cm, font=\color{black}] {\# grad. eval. to converge};
                    
                    \node [right=of img2, xshift=-.65cm](img3){\includegraphics[trim={1.85cm .9cm .97cm .9cm},clip,width=.3\linewidth]{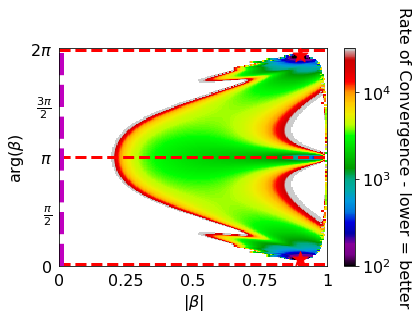}};
                    \node[above=of img3, node distance=0cm, xshift=-.3cm, yshift=-1.3cm,font=\color{black}] {\small \ref{eq:alt_update} for \# steps};
                    \node[right=of img3, node distance=0cm, rotate=270, xshift=-1.5cm, yshift=-.9cm, font=\color{black}] {\# steps to converge};
                \end{tikzpicture}
                \vspace{-0.025\textheight}
                \caption{
                    We show many steps and gradient evaluations, both simultaneous and alternating complex momentum on a Dirac-GAN take for a set solution distance.
                    We fix step size $\lr \!=\! \num{0.1}$ as in Figure~\ref{fig:trajectories_limited}, while varying the phase and magnitude of our momentum $\momCoeff  \!=\! |\momCoeff| \exp(i \arg(\momCoeff))$.
                    There is a {\color{red}red} star at the optima, dashed {\color{red}red} lines at real $\momCoeff$, and a dashed {\color{magenta}magenta} line for simultaneous or alternating gradient descent.
                    We only display color for convergent setups.
                    \emph{Left:} Simultaneous complex momentum (\ref{eq:simul_update}).
                    This is the same as Figure~\ref{fig:det_minmax_expAvg_simul_phase_tune}, which we repeat to contrast with alternating updates.
                    There are no real-valued $\momCoeff$ that converge for this -- or any -- $\lr$ with simultaneous updates ~\citep{gidel2018negative}.
                    Simultaneous updates can parallelize gradient computation for all players at each step, thus costing only one gradient evaluation per step for many deep learning setups.
                    The best rate of convergence per step and gradient evaluation is $\approx {\color{red}\num{.955}}$.
                    \emph{Middle:} Alternating complex momentum (\ref{eq:alt_update}), where we show how many gradient evaluations -- as opposed to steps -- to reach a set solution distance.
                    Alternating updates are bottlenecked by waiting for first player's update to compute the second players update, effectively costing two gradient evaluations per step for many deep learning setups.
                    Negative momentum can converge here, as shown by \citet{gidel2018negative}, but the best momentum is still complex.
                    Also, Alternating updates can make the momentum phase $\arg(\momCoeff)$ choice less sensitive to our convergence.
                    The best rate of convergence per gradient evaluation is $\approx {\color{red}\num{.965}}$.
                    \emph{Right:} \ref{eq:alt_update}, where we show how many steps to reach a set solution distance.
                    The best rate of convergence per step is $\approx {\color{red}\num{.931}}$.
                    \textbf{Takeaway:} If we can parallelize computation of both players gradients we can benefit from \ref{eq:simul_update}, however if we can not then \ref{eq:alt_update} can converge more quickly and for a broader set of optimizer parameters.
                    In any case, the best solution uses a complex momentum $\momCoeff$ for this $\lr$.
                    %
                }
                \label{fig:det_minmax_expAvg_alt_phase_tune}
                \vspace{-0.02\textheight}
            \end{figure}
            \newcommand{\ytickspaceApp}{0.049\textwidth}
            \begin{figure}
                \vspace{-0.005\textheight}
                \centering
                \begin{tikzpicture}
                    \centering
                    \node (img1)[]{\hspace{-1.5cm}\includegraphics[trim={.5cm .5cm .37cm .475cm},clip,width=.39\linewidth]{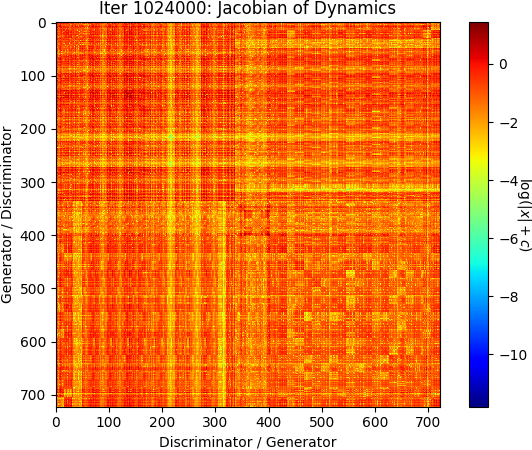}};
                    \node[left=of img1, node distance=0cm, rotate=90, xshift=2.0cm, yshift=.675cm, font=\color{black}] {Joint-gradient $\bothGrad$ index $l$};
                    \node[right=of img1, node distance=0cm, rotate=270, xshift=-2.4cm, yshift=-.75cm, font=\color{black}] {$\log(|a_{lk}| + \epsilon)$ for $a_{lk}$ in $\nabla_{\!\bothParam} \bothGrad^{j}$};
                    \node[below=of img1, node distance=0cm, xshift=-1.1cm, yshift=1.2cm,font=\color{black}] {Joint-parameter $\bothParam$ index $k$};
                    \node[above=of img1, node distance=0cm, xshift=-.8cm, yshift=-1.2cm,font=\color{black}] {Jacobian of joint-gradient $\nabla_{\!\bothParam} \bothGrad^{j}$ for GAN};
                    
                    \node (img2)[right=of img1, xshift=0.25cm]{\includegraphics[trim={.5cm .5cm .0cm .475cm},clip,width=.43\linewidth]{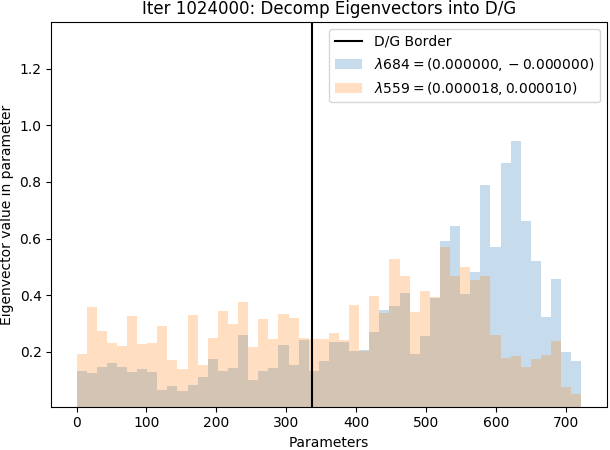}};
                    \node[left=of img2, node distance=0cm, rotate=90, xshift=2.25cm, yshift=-0.825cm, font=\color{black}] {Abs of eigenvector component};
                    \node[below=of img2, node distance=0cm, xshift=-.1cm, yshift=1.2cm,font=\color{black}] {Index into $\mathbf{v}$ -- first $\num{337}$ are for D's params};
                    \node[above=of img2, node distance=0cm, xshift=-.1cm, yshift=-1.0cm,font=\color{black}] {Eigenvector $\mathbf{v}$ components for eigenvalues $\eigval \in \spectrum(\nabla_{\!\bothParam} \bothGrad)$};
                    
                    
                    \node (img3)[below=of img1, yshift=-.4cm, xshift=0]{\includegraphics[trim={1.3cm .5cm 1.8cm .475cm},clip,width=.41\linewidth]{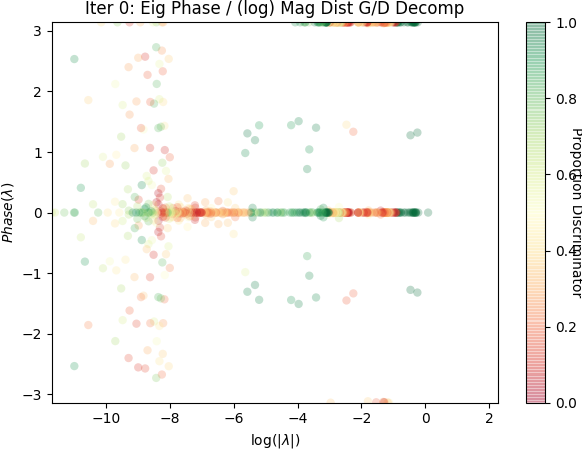}};
                    \node (img4)[right=of img3, xshift=\shiftHeatmapx]{\includegraphics[trim={1.3cm .5cm .9cm .475cm},clip,width=.44\linewidth]{images/16_gan_example/eig_mag_phase_distribution_decomposition_1024000.png}};
                    
                    \node (img3ytick)[left=of img3, node distance=0cm, rotate=90, xshift=2.8cm, yshift=-0.825cm, font=\color{black}] {-$\pi$ \hspace{\ytickspaceApp} -$\frac{\pi}{2}$ \hspace{\ytickspaceApp} $\num{0}$ \hspace{\ytickspaceApp} \hphantom{-}$\frac{\pi}{2}$ \hspace{\ytickspaceApp} \hphantom{-}$\pi$};
                    \node[left=of img3ytick, node distance=0cm, rotate=90, xshift=5.05cm, yshift=-0.43cm, font=\color{black}] {Phase of eigenvalue $\arg(\eigval)$};
                    
                    \node (img4colortick)[right=of img4, node distance=0cm, rotate=270, xshift=-2.8cm, yshift=-0.83cm, font=\color{black}] {disc. \hspace{0.08\textwidth} unsure \hspace{0.09\textwidth} gen.};
                    \node[right=of img4colortick, node distance=0cm, rotate=270, xshift=-5.3cm, yshift=-.55cm, font=\color{black}] {Does eigenvector point at a player?};

                    \node[below=of img3, node distance=0cm, xshift=3.65cm, yshift=1.2cm, font=\color{black}] {Log-magnitude of eigenvalue $\log(|\eigval|)$};
                    
                    \node[above=of img3, node distance=0cm, xshift=0.0cm, yshift=-1.25cm, font=\color{black}] {Start of training};
                    \node[above=of img4, node distance=0cm, xshift=0.0cm, yshift=-1.25cm, font=\color{black}] {End of training};
                    
                    \node[above=of img3, node distance=0cm, xshift=3.2cm, yshift=-.6cm, font=\color{black}] {Log-polar graph of the spectrum of Jacobian of the joint-gradient $\spectrum(\nabla_{\!\bothParam} \bothGrad^{j})$ throughout training};
                \end{tikzpicture}
                \vspace{-0.02\textheight}
                \caption{
                    These plots investigate the spectrum of the Jacobian of the joint-gradient for the GAN in Figure~\ref{fig:gan_spectrum_main} through training.
                    The spectrum is key for bounding convergence rates in learning algorithms.
                    \newline
                    \emph{Top left:} The Jacobian $\nabla_{\bothParam} \bothGrad$ for a GAN on a $\num{2}$D mixture of Gaussians with a two-layer, fully-connected $\ganSize$ hidden unit discriminator (D) and generator (G) at the end of training.
                    In the concatenated parameters $\bothParam \in \mathbb{R}^{723}$, the first $\num{337}$ are for D, while the last $\num{386}$ are for G.
                    We display the $\log$ of the absolute value of each component plus $\epsilon = 10^{-10}$.
                    The upper left and lower right quadrants are the Hessian of D and G's losses respectively.
                    \newline
                    \emph{Top Right:} We visualize two randomly sampled eigenvectors from $\nabla_{\bothParam} \bothGrad$.
                    The first part of the parameters is for the discriminator, while the second part is for the generator.
                    Given an eigenvalue with eigenvector $\mathbf{v}$, we roughly approximate attributing eigenvectors to players by calculating how much of it lies in D's parameter space with $\frac{\| \mathbf{v}_{1:|\textnormal{D}|} \|_1}{\| \mathbf{v} \|_1} = \frac{\| \mathbf{v}_{1:337} \|_1}{\| \mathbf{v} \|_1}$.
                    If this ratio is near $\num{1}$ (or $\num{0}$) and say \emph{the eigenvector mostly points at D (or G)}.
                    The blue eigenvector mostly points at G, while the orange eigenvector is unclear.
                    Finding useful ways to attribute eigenvalues to players is an open problem.
                    \newline
                    \emph{Bottom:} The spectrum of the Jacobian of the joint-gradient $\spectrum(\nabla_{\bothParam} \bothGrad^{j})$ is shown in log-polar coordinates, because it is difficult to see structure when graphing in Cartesian (i.e., $\Re$ and $\Im$) coordinates, due to eigenvalues spanning orders of magnitude, while being positive and negative.
                    The end of training is when we stop making progress on the log-likelihood.
                    We have imaginary eigenvalues at $\arg(\eigval) = \pm \nicefrac{\pi}{\num{2}}$, positive eigenvalues at $\arg(\eigval) = \num{0}$, and negative eigenvalues at $\arg(\eigval) = \pm \pi$.
                    \newline
                    \emph{Takeaway:}
                    There is a banded structure for the coloring of the eigenvalues that persists through training.
                    We may want different optimizer parameters for the discriminator and generator, due to asymmetry in their associated eigenvalues.
                    Also, the magnitude of the eigenvalues grows during training, and the $\arg$s spread out indicating the game can change eigenstructure near solutions.
                }\label{fig:gan_decomp}
            \end{figure}

        \newcommand{\latentDim}{\num{4}}
        \newcommand{\mixtureNum}{\num{8}}
        \newcommand{\numHiddenUnits}{\num{256}}
        \newcommand{\numIterations}{\num{100000}}
        \newcommand{\batchSizeSmallGAN}{\num{100000}}
        \newcommand{\dataDim}{\num{2}}
        \vspace{-0.01\textheight}
        \subsection{Training GANs on $\num{2}$D Distributions}\label{sec:app_2d_gan}
        \vspace{-0.01\textheight}
            For $\num{2}$D distributions, the data is generated by sampling from a mixture of $\mixtureNum$ Gaussian distributions, which are distributed uniformly around the unit circle.
            
            
            For the GAN, we use a fully-connected network with $\num{4}$ hidden ReLU~\citep{hahnloser2000digital} layers with $\numHiddenUnits$ hidden units.
            We chose this architecture to be the same as \citet{gidel2018negative}.
            Our noise source for the generator is a $\num{4}$D Gaussian.
            We trained the models for $\numIterations$ iterations.
            The performance of the optimizer settings is evaluated by computing the negative log-likelihood of a batch of $\batchSizeSmallGAN$ generated $\dataDim$D samples.
            
            \newcommand{\ganHeatmapWidth}{.46\textwidth}
            \newcommand{\legendWidth}{.025\textwidth}
            \begin{figure}
                \vspace{-0.01\textheight}
                \centering
                \begin{tikzpicture}
                    \centering
                    \node (img1){\includegraphics[trim={4.2cm 1.75cm 8.85cm 2.57cm},clip, width=\ganHeatmapWidth]{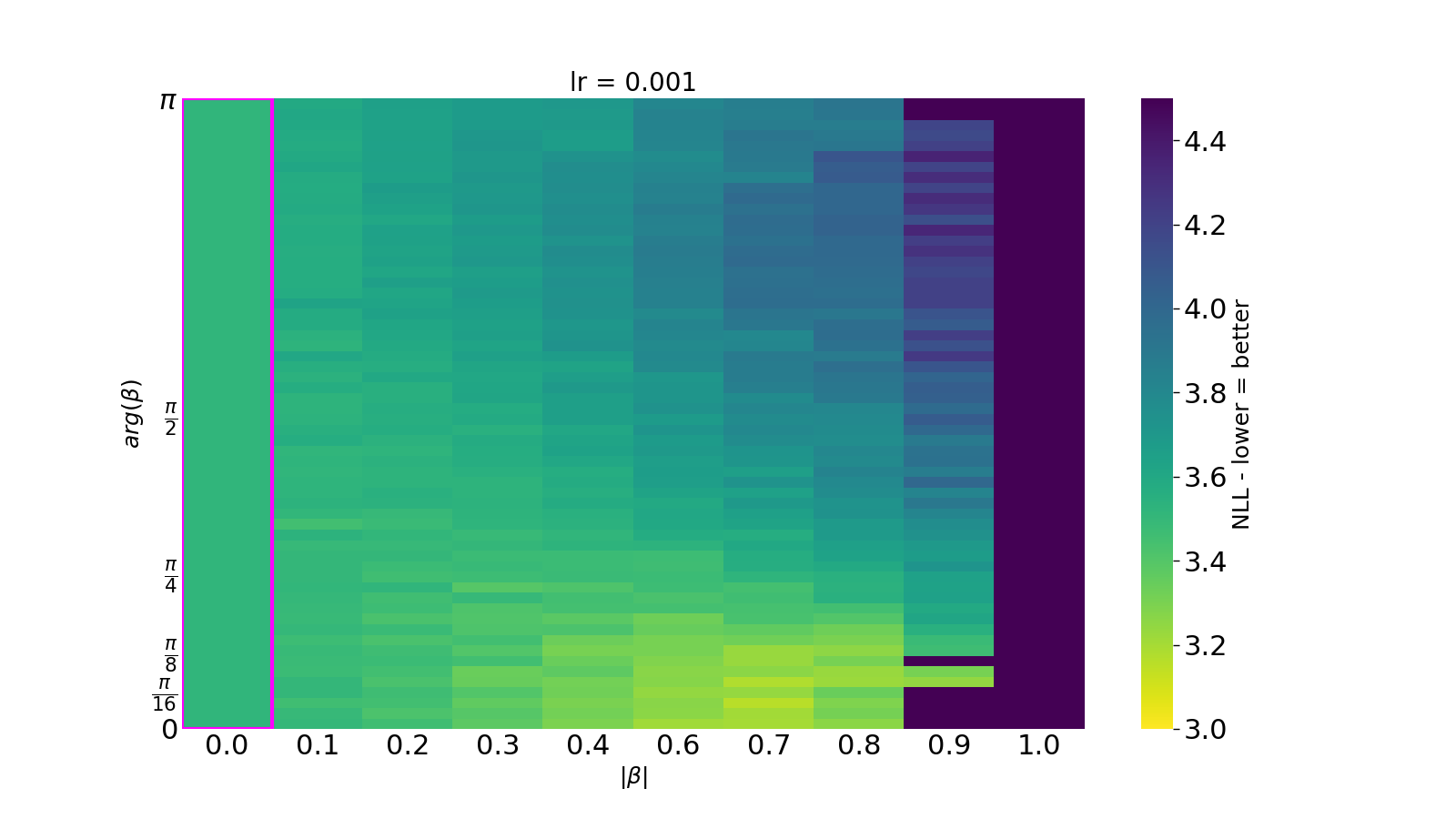}};
                    \node[left=of img1, node distance=0cm, rotate=90, xshift=2.25cm, yshift=-.9cm, font=\color{black}] {Momentum phase $\arg(\momCoeff)$};
                    \node[above=of img1, node distance=0cm, xshift=0cm, yshift=-1.1cm,font=\color{black}] {Step size $\lr = \num{0.001}$};
                    \node(img2cap)[below=of img1, node distance=0cm, xshift=3.3cm, yshift=1.1cm,font=\color{black}] {Momentum Magnitude $| \momCoeff |$};
                    
                    \node (img2)[right=of img1, node distance=0cm, xshift=-1.4cm]{\includegraphics[trim={5.1cm 1.75cm 6.4cm 2.57cm},clip,width=\ganHeatmapWidth + 0.02\textwidth]{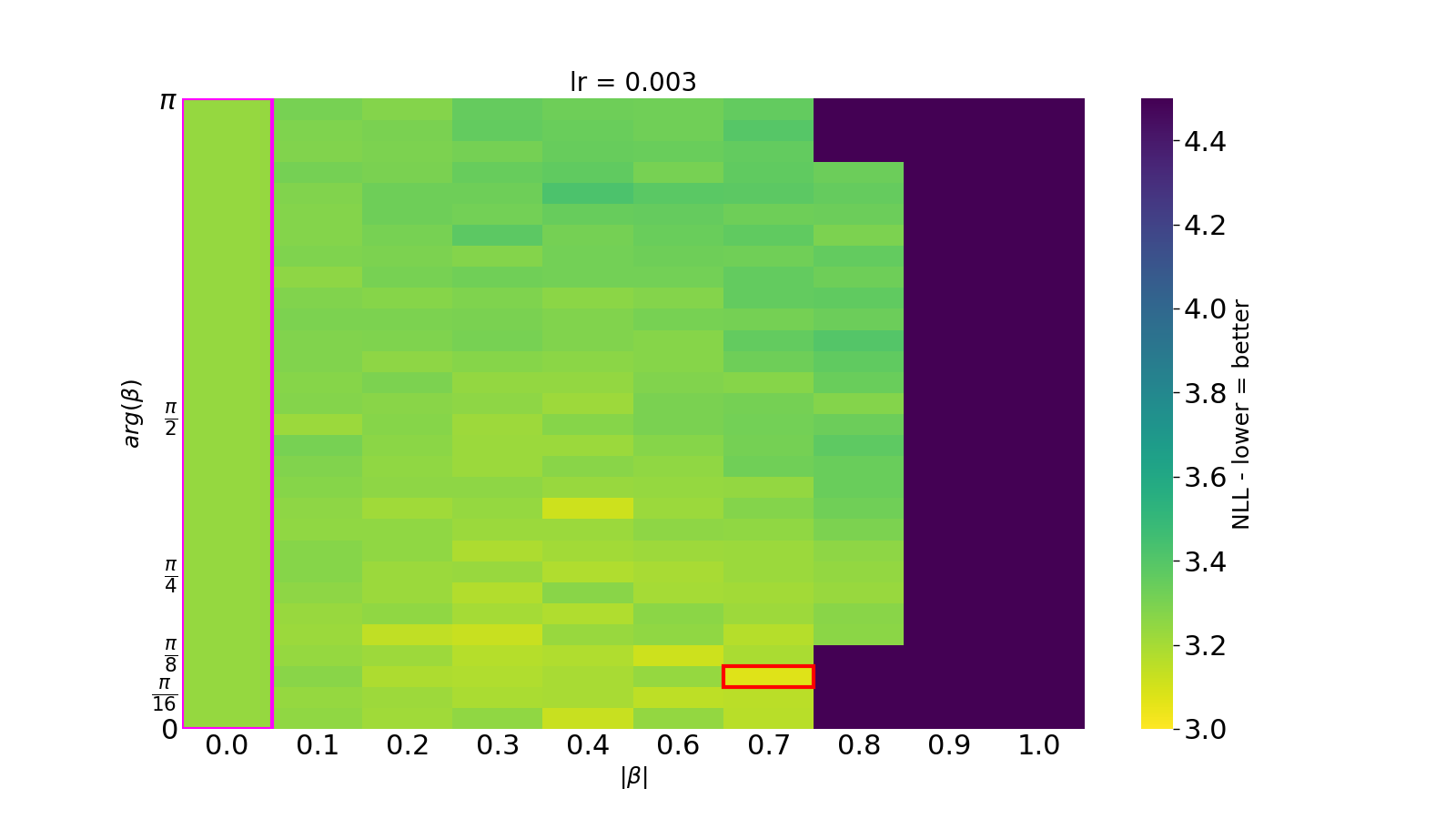}};
                    \node[above=of img2, node distance=0cm, xshift=0.0cm, yshift=-1.1cm,font=\color{black}] {Step Size $\lr = \num{0.003}$};
                    \node[right=of img2, node distance=0cm, rotate=270, xshift=-2.15cm, yshift=-.85cm, font=\color{black}] {NLL, lower = better};
                \end{tikzpicture}
                \vspace{-0.03\textheight}
                \caption{
                    Heatmaps of the negative log-likelihood (NLL) for tuning $\arg(\momCoeff), |\momCoeff|$ with various fixed $\lr$ on a $\num{2}$D mixture of Gaussians GAN.
                    We highlight the best performing cell in {\color{red}red}, which had $\arg(\momCoeff) \approx \nicefrac{\pi}{\num{8}}$.
                    Runs equivalent to alternating SGD are shown in a {\color{magenta}magenta} box.
                    We compare to negative momentum with alternating updates as in \citet{gidel2018negative} in the top row with $\arg(\momCoeff) = \pi$.
                    \emph{Left}: Tuning the momentum with $\lr = \num{0.001}$.
                    \emph{Right}: Tuning the momentum with $\lr = \num{0.003}$.
                }
                \label{fig:gan_heatmaps}
            \end{figure}
        \begin{figure}
            \vspace{-0.01\textheight}
            \centering
            \begin{subfigure}{.45\linewidth}
                \includegraphics[trim={1.25cm 1.25cm 1.25cm 1.25cm},clip,width=.99\textwidth]{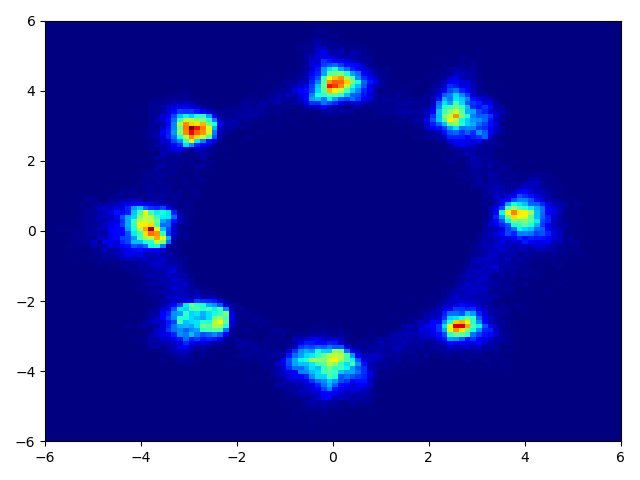}
                \newsubcap{Mixture of Gaussian samples from GAN with the best hyperparameters from the heatmaps in Appendix Figure~\ref{fig:gan_heatmaps}}
                \label{fig:2d_samples}
            \end{subfigure}
            \hspace{0.05\linewidth}
            \begin{subfigure}{.45\linewidth}
                \includegraphics[trim={0cm 0cm 3.6cm 6cm},clip,width=.99\textwidth]{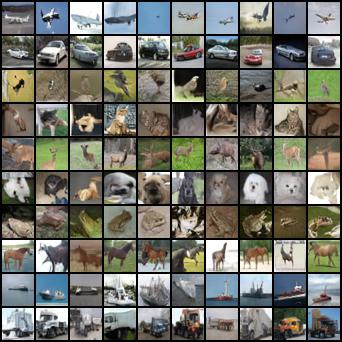}
                \newsubcap{Class-conditional CIFAR-10 samples from GAN with the best hyperparameters from the heatmap in \ref{fig:biggan_inception_heatmap}}
                \label{fig:biggan_samples}
            \end{subfigure}
            %
                    
            \vspace{-0.01\textheight}
        \end{figure}
        \begin{figure}
            \centering
            \begin{subfigure}{0.45\linewidth}
                \vspace{-0.01\textheight}
                \centering
                \begin{tikzpicture}
                    \centering
                    \node (img1){\includegraphics[trim={.875cm .9cm 0cm 0cm},clip,width=.84\linewidth]{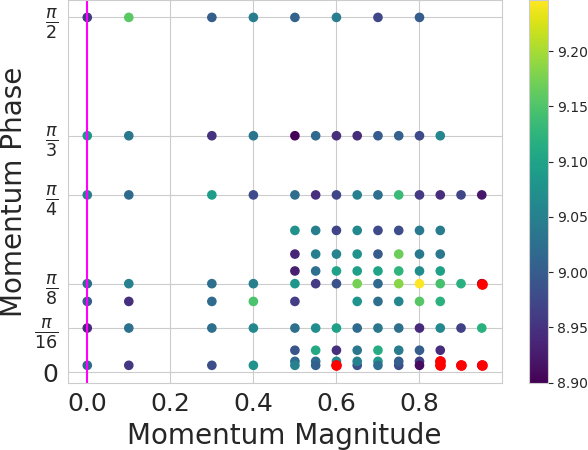}};
                    \node[left=of img1, node distance=0cm, rotate=90, xshift=2.2cm, yshift=-.9cm, font=\color{black}] {Momentum phase $\arg(\momCoeff_1)$};
                    \node[below=of img1, node distance=0cm, xshift=-.1cm, yshift=1.2cm,font=\color{black}] {Momentum magnitude $| \momCoeff_1 |$};
                    \node[right=of img1, node distance=0cm, rotate=270, xshift=-2.5cm, yshift=-.9cm, font=\color{black}] {Inception score, higher = better};
                \end{tikzpicture}
                \newsubcap{
                    The inception score (IS) for a grid search on $\arg(\momCoeff_1)$ and $|\momCoeff_1|$ for training BigGAN on CIFAR-10 with the Adam variant in Algorithm~\ref{alg:complex_adam}.
                    The $\momCoeff_1$ is complex for the discriminator, while the generator's optimizer is fixed to author-supplied defaults.
                    {\color{red}Red} points are runs that failed to train to the minimum IS in the color bar.
                    The vertical {\color{magenta}magenta} line denotes runs equivalent to alternating SGD.
                    Negative momentum failed to train for any momentum magnitude $|\momCoeff_1| > .5$, so we do not display it for more resolution near values of interest.
                }
                \label{fig:biggan_inception_heatmap}
                \vspace{-0.01\textheight}
            \end{subfigure}
            \hspace{0.05\linewidth}
            \begin{subfigure}{0.45\linewidth}
                \centering
                \begin{tikzpicture}
                    \centering
                    \node (img1){\includegraphics[trim={.875cm .9cm 0cm 0cm},clip,width=.84\linewidth]{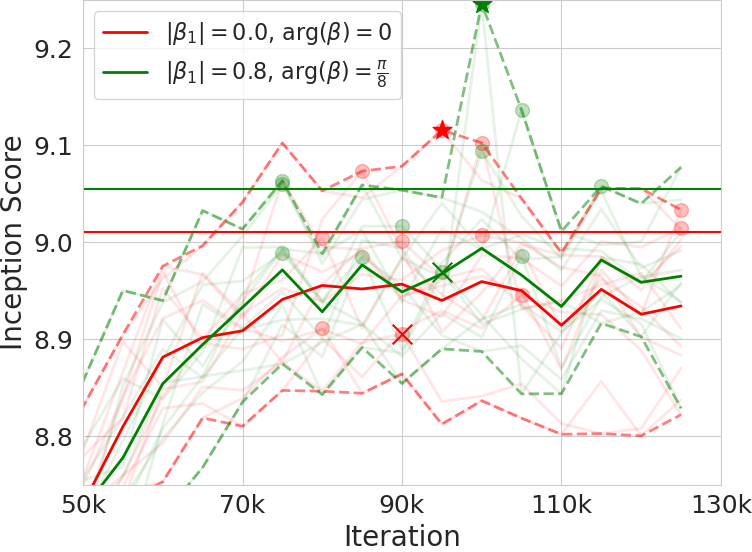}};
                    \node[left=of img1, node distance=0cm, rotate=90, xshift=1.5cm, yshift=-.5cm, font=\color{black}] {Inception score (IS)};
                    \node[below=of img1, node distance=0cm, xshift=-.1cm, yshift=1.0cm,font=\color{black}] {Iteration};
                \end{tikzpicture}
                \newsubcap{
                    We compare the best optimization parameters from grid search Figure~\ref{fig:biggan_inception_heatmap} for our complex Adam variant (i.e., Algorithm~\ref{alg:complex_adam}) shown in {\color{green}green}, with the author provided values shown in {\color{red}red} for the CIFAR-10 BigGAN over $\num{10}$ seeds.
                    A star is displayed at the best IS over all runs, a cross is displayed at the worst IS over all runs, while a circle is shown at the best IS for each run.
                    Dashed lines are shown at the max/min IS over all runs at each iteration, low-alpha lines are shown for each runs IS, while solid lines are shown for the average IS over all seeds at each iteration.
                    The results are summarized in Table~\ref{tab:biggan}.
                }
                    
                \label{fig:biggan_inception_runs}
            \end{subfigure}
            
        \end{figure}
\end{document}